\documentclass[conference]{IEEEtran}
\IEEEoverridecommandlockouts

\usepackage{cite}
\usepackage{amsmath,amssymb,amsfonts}
\usepackage{algorithmic}
\usepackage{graphicx}
\usepackage{textcomp}
\usepackage{xcolor}
\usepackage{enumitem}

\usepackage{bm}
\usepackage{amssymb}
\usepackage{amsmath}
\usepackage{amsthm}
\newtheorem{theorem}{Theorem}
\newtheorem{lemma}[theorem]{Lemma}

\newtheorem{proposition}[theorem]{Proposition}

\newtheorem{assumption}{Assumption}
\usepackage{algorithm,algorithmic}
\usepackage{url}

\usepackage{newfloat}
\usepackage{listings}

\def\BibTeX{{\rm B\kern-.05em{\sc i\kern-.025em b}\kern-.08em
    T\kern-.1667em\lower.7ex\hbox{E}\kern-.125emX}}
\begin{document}

\title{Collaborative Bayesian Optimization via Wasserstein Barycenters}
% \\

\author{\IEEEauthorblockN{Donglin Zhan, Haoting Zhang, Rhonda Righter, Zeyu Zheng, and James Anderson
\thanks{Donglin Zhan and James Anderson are with the Department of Electrical Engineering at Columbia University, New York, USA. Emails: \texttt{\{donglin.zhan,james.anderson\}@columbia.edu}. Haoting Zhang, Rhonda Righter, and Zeyu Zheng are with the Department of Industrial Engineering and Operation Research at the University of California at Berkeley, Berkeley, USA. Emails: \texttt{\{haoting\_zhang,rrighter,zyzheng\}@berkeley.edu}.}}

}

\maketitle

\begin{abstract}
Motivated by the growing need for black-box optimization and data privacy, we introduce a collaborative Bayesian optimization (BO) framework that addresses both of these challenges. In this framework agents work collaboratively to optimize a function they only have oracle access to. In order to mitigate against communication and privacy constraints, agents are not allowed to share their data but can share their Gaussian process (GP) surrogate models. To enable collaboration under these  constraints, we construct a central model to approximate the objective function by leveraging the concept of Wasserstein barycenters of GPs. This central model integrates the shared models without accessing the underlying data. A key aspect of our approach is a collaborative acquisition function that balances exploration and exploitation, allowing for the optimization of decision variables collaboratively in each iteration. We prove that our proposed algorithm is asymptotically consistent and that its implementation via Monte Carlo methods is numerically accurate. Through numerical experiments, we demonstrate that our approach outperforms other baseline collaborative frameworks and is competitive with centralized approaches that do not consider data privacy.
\end{abstract}

\section{Introduction}

In numerous engineering applications, including materials design \cite{frazier2016bayesian}, control engineering\cite{ghoreishi2020bayesian}, robotics \cite{martinez2009bayesian} and machine learning parameter tuning \cite{snoek2012practical}, one must solve optimization problems where either the objective function is costly to evaluate, or, is unknown and only oracle access is available.  A widely used method in this setting is Bayesian optimization (BO), which employs a ``surrogate model'' to approximate the unknown objective function \cite{frazier2018bayesian}. One of the most widely-used surrogates for BO is the Gaussian process (GP) model \cite{williams2006gaussian}, which assumes that the objective function is a realization of a GP. In conjunction with the surrogate model, an ``acquisition function'' uses predictions from the surrogate model to determine the next sample point. The BO algorithm is implemented by iteratively collecting observations, updating the model, and optimizing the acquisition function to select the decision variable to collect another observation.

In various control and machine learning scenarios, multiple agents collaborate to optimize an unknown function. For instance, agents may need to collaboratively fine-tune controller parameters in distributed control architectures, such as multi-agent robotic systems \cite{liang2021secure} or decentralized energy management systems \cite{chen2012optimal}. Another example would be fine-tuning the hyperparameters of several different neural network architectures in a parallel computing environment \cite{wu2016parallel}. To facilitate this collaboration, observations at evaluated points are shared to update a joint surrogate model. The acquisition function is designed to select multiple points for parallel evaluation by multiple agents in the next iteration. In many cases, data privacy might be a concern, in which case agents are not allowed to share the data directly when updating the surrogate model. An illustrative example is found in autonomous vehicle platooning, where each vehicle locally optimizes its control parameters based on sensor data that cannot be directly shared due to privacy and security regulations \cite{pan2022privacy}. Similarly, in distributed power grids, individual units optimize local control parameters without exchanging sensitive consumption or operational data directly, adhering strictly to privacy and security constraints \cite{du2018distributed, zhou2019differential}.

Data privacy concerns have been discussed extensively in the literature on federated learning (FL), c.f.,  \cite{mcmahan2017communication,li2020federated}. In FL, instead of sharing the data, each agent learns a model from its own dataset and shares the resulting model. The central model is constructed from the models shared by the agents, aiding in approximating the unknown function and facilitating downstream applications, including optimization. Existing federated learning algorithms primarily rely on parametric models, such as neural networks and kernel regression models \cite{zhao2018federated,yuan2020federated}. For example, FedAvg is a foundational algorithm that directly averages local updates to achieve a global model \cite{mcmahan2017communication}, and FedProx addresses agent heterogeneity by adding a proximal term to stabilize the learning process across diverse data distributions \cite{li2020federated}.
In these cases, agents share the learned parameters representing their surrogate models, and the central model is constructed by adaptively averaging these parameters using various methodologies. In contrast, the GP models used in BO algorithms are nonparametric, which poses a challenge for existing methods to construct a central model from the surrogate models shared by agents. In this work, we ask the question: \emph{Is it possible to collaboratively learn in a Bayesian Optimization setting without sharing data?}

\textbf{Contribution.} We propose a framework for collaborative Bayesian Optimization (BO) with data privacy considerations. In this framework, agents approximate the objective function using GP models. Instead of sharing data, the agents share these GP models with a (central) server. The server then constructs a central GP model by leveraging the concept of the Wasserstein barycenter. Specifically, we treat GPs as probability measures and represent the central model as the probability measure that minimizes the squared Wasserstein distance to the local GP models. This central model integrates the local models without directly sharing data. Furthermore, the central model remains a GP and provides explicit uncertainty quantification. We propose a collaborative acquisition function to select decision variables for agents to collect observations in parallel. The proposed acquisition function not only addresses the exploration-exploitation trade-off during the optimization procedure but also leverages the central model while maintaining differences through consideration of the local models. Our main contributions are as follows: 
\begin{itemize}[leftmargin=*]
    \item A collaborative BO framework with data privacy considerations, where the central model is constructed as the Wasserstein barycenter of GPs and remains a GP. This GP structure allows for explicit uncertainty quantification and facilitates the exploitation-exploration trade-off during the BO implementation.

    \item A collaborative acquisition function that selects decision variables for agents to collect observations in parallel, focusing on the collaborative knowledge gradient (Co-KG) function and proving the consistency of the framework based on Co-KG. We use a Monte Carlo method to approximate Co-KG, proving the consistency of the approximation.

    \item  Our experimental results show that Co-KG achieves the best performance compared to other collaborative acquisition functions. We provide practical suggestions for selecting the hyperparameters for Co-KG implementation. Results indicate that our approach not only outperforms existing baseline approaches but also achieves comparable performance to BO algorithms where agents can directly share data to update the GP model without data privacy concerns.
\end{itemize}

\subsection{Related Literature}
BO has been widely applied to solve black-box optimization problems, including materials/engineering design \cite{frazier2016bayesian} and parameter tuning \cite{snoek2012practical}. The BO approach involves a statistical surrogate model, typically a GP, learned from observations at evaluated points, and an acquisition function, constructed from the surrogate, for deciding the next evaluation point. The algorithm iteratively collects observations, updates the model, and optimizes the acquisition function, with common acquisition functions including expected improvement \cite{jones1998efficient}, knowledge gradient \cite{frazier2009knowledge} and upper confidence bound \cite{srinivas2010gaussian}. Despite its popularity, BO approaches have been largely restricted to moderate-dimensional problems due to the computational complexities brought by Gaussian processes (GPs). Thus, large streams of work have focused on high-dimensional BO, where the approximation of the covariance matrix is employed or an additive/sparse structure of the GP is imposed; see \cite{wang2018batched,jimenez2023scalable}. There have also been many recent contributions to multi-objective/task BO \cite{daulton2022robust,lin2022preference}, budgeted BO \cite{astudillo2021multi}, and multi-fidelity BO \cite{kandasamy2017multi}. Additionally, BO within a collaborative framework has also been explored. Work in \cite{dai2020federated} casts BO in a FL setting, where agents are randomly selected at each iteration, and the central GP model is approximated by a parametric model. \cite{yue2025collaborative} considers optimizing a weighted mean of acquisition functions based on multiple GPs to facilitate collaboration. \cite{zerefa2025distributed} enhances the efficiency of agents with constrained communication graphs via distributed Thompson sampling. %\ja{\textbf{[We need to say how our work differs...]}}

\section{Problem Statement \& Main Procedure}
We consider the black-box optimization problem:  $$\max_{x\in\mathcal{X}}~f(x),$$ where $x\in\mathbb{R}^d$ is the decision variable, $\mathcal{X}$ is the continuous feasible set, and $f: \mathcal{X}\mapsto \mathbb{R}$ is a black-box objective function. Each evaluation of $f(x)$ is expensive, so the exploitation-exploration trade-off is a concern. In addition, each evaluation includes observation noise, i.e., for the  $t$-th {sample} we observe $y_t = f\left(x_t\right)+\epsilon_t.$ and {we assume a Gaussian noise model where $\epsilon_t\stackrel{i.i.d.}{\sim}\mathcal{N}\left(0,\sigma ^2_\epsilon\right)$.}

We use a Gaussian process (GP) to model the objective function. That is, the objective function $f(x) \sim \mathcal{GP}\left(0,K\left(x,x'\right)\right),$ where $K\left(x,x'\right) \doteq \operatorname{Cov}\left(f\left(x\right),f\left(x'\right)\right)$ is a pre-specified kernel function that quantifies the similarity of the surrogate model $f$ evaluated at different decion variables $x$ and $x'$. A common selection is the radial basis function (RBF) kernel 
\begin{equation}
    \label{eq.rbf}K_{\operatorname{RBF}}\left(x,x'\right) = \exp\left\{-\frac{\left\|x-x'\right\|^2}{\sigma^2}\right\},
\end{equation}
where $\left\|\,\cdot\,\right\|$ denotes the Euclidean norm of a vector and $\sigma^2$ is a user-specified hyperparameter. The kernel function $K\left(x,x'\right)$ is associated with an operator $\Phi_{K}: L^2(\mathcal{X}) \rightarrow L^2(\mathcal{X})$ with
$$
[\Phi_{K} \phi](x)=\int_\mathcal{X} K(x, s) \phi(s)\, \mathrm{d} s\quad \forall \phi \in L^2(\mathcal{X}),
$$
where $L^2(\mathcal{X})$ is the space of functions that are square-integrable over the domain $\mathcal{X}$. We refer to \cite{seeger2004gaussian,van2008reproducing} for detailed discussions of the kernel operator.

The GP model enjoys explicit uncertainty quantification associated with the function approximation. Conditional on the historical dataset $\tilde {\mathcal{S}} = \left \{ \left(x_1,y_1 \right),\ldots,\left(x_t,y_t  \right)\right \}$, the objective function $f(x)$ remains a GP model: 
\begin{equation}
    \label{postgp}
    f(x)\mid\tilde{\mathcal{S}} \sim \mathcal{GP}\left(\tilde{\mu}\left(x\right),\tilde{K}\left(x,x'\right)\right),
\end{equation}
where 
%\begin{equation}
  \begin{align}
        &\tilde{\mu}\left(x\right) = \bm{K}_t\left ( x\right )^{\top} \left(\tilde{\bm{K}}_t+\sigma^2_\epsilon\bm{I}_t\right)^{-1} \tilde{\bm{y}}_t,\label{eq.gp} \\
        &\tilde{K}\left(x,x'\right)\! =\! K\left ( x, x'\right )   \!-\! \bm{K}_t\left ( x\right )^{\top} \left(\tilde{\bm{K}}_t+\sigma^2_\epsilon\bm{I}_t\right)^{-1}\!\!\bm{K}_t\left ( x'\right ). \nonumber
    \end{align}
%\end{equation}
Here $\tilde{\bm{y}}_t \!\!=\!\! \left ( y_1,\ldots,y_t \right )^{\top}\!\in\! \mathbb{R}^t$ denotes the {aggregated}  observation {vector}; $\tilde{\bm{K}}_t\in\mathbb{R}^{t\times t}$ is the kernel matrix with $\left(\tau,\tau'\right)$-th entry $K\left(x_{\tau},x_{\tau'}\right )$; and $\bm{K}_t\left ( x\right ) \! \!= \!\left ( K\left(x,x_1\right ),K\left(x,x_2\right ),\ldots,K\left(x,x_t\right ) \right )^{\top}\!\!\in\!\mathbb{R}^t$ is the vector of kernel function values between $x$ and $\left\{x_{\tau}\right\}_{\tau=1}^{t}$. With the explicit inference of $f(x)$ as in (\ref{postgp}), an acquisition function is constructed to account for both exploitation and exploration. A popular choice of  acquisition function is the knowledge gradient (KG) \cite{frazier2009knowledge}, which is defined by
\begin{equation}\label{eq.initialkg}
    \alpha_{\text{KG}}\left ( x \right )  = \mathbb{E}_{\tilde{\mathcal{S}}}\left[\max_{x'\in\mathcal{X}}~\mathbb{E}\left[f\left ( x' \right )\mid \tilde{\mathcal{S}}_x^*\right]\right],
\end{equation}
where $\tilde{\mathcal{S}}$ represents the historical observations, and $\tilde{\mathcal{S}}_x^* = \tilde{\mathcal{S}}\cup\left\{\left(x,y(x)\right)\right\}$ is the updated dataset if an additional observation $y(x)$ is decided to be collected at the point $x$. In this manner, the posterior mean $\mathbb{E}\left[f\left ( x' \right )\mid \tilde{\mathcal{S}}_x^*\right]$ serves as an approximation of the objective function with a future observation $(x,y(x))$ taken into consideration, and the acquisition function (\ref{eq.initialkg}) is maximized to select the next decision variable.

In this work, we specifically consider a collaborative framework, where there are $N$  agents  {independently, locally, and and in parallel} sampling decision variables  to maximize an identical black-box objective function $f(x)$. At each iteration, these agents communicate through a server. The server collects information (which will be specified later) from the agents and then decides which decision variables to sample for them. We denote the $t$-th data pair collected by the $n$-th agent as $\left(x_{n;t}, y_{n;t}\right)$.  {In the setting where data privacy is not an issue, } all such pairs $\left(x_{n;t}, y_{n;t}\right)$ are sent to the server. The server then constructs a central GP model from these samples.  {This is precisely the  parallel BO or batch BO \cite{wu2016parallel,daulton2021parallel} algorithm, where the server selects a batch of $N$ decision variables $x_{n;t+1}$'s at each iteration, and there is no distinction between the agents.}

 {In contrast, }our work assumes data privacy \emph{is} a concern. In other words, the data $\left(x_{n;t},y_{n;t}\right)$ collected by each agent, is not allowed to be sent to the server. Instead, each local model, updated by the data collected by each agent, is shared. Specifically, we denote the historical dataset the $n$-th agent has collected as $\tilde{\mathcal{S}}_{n}$  {for $i=1,\hdots, N$.} We eliminate the dependency of the size of each dataset for notational simplicity. We denote the posterior of the black box function as $\tilde{f}_{n}\doteq f(x)\mid\tilde{\mathcal{S}}_{n}$, which is characterized by the associated posterior mean function $\tilde{\mu}_n(x)$ and the posterior kernel function $\tilde{K}_n\left(x,x'\right)$ as in (\ref{eq.gp}). We implement the collaborative framework by sending the posterior GPs from the agents to the server. Thus, in each iteration, the server receives a set of GP models $\left\{\tilde{f}_{1},\ldots,\tilde{f}_{N}\right\}$ from the agents. To attain a central model, we use the Wasserstein barycenter~\cite{puccetti2020computation} of the GPs.  {Informally, he Wasserstein barycenter can be thought of as a weighted average of probability measures}. The Wasserstein barycenter is robust to outliers and captures the central tendency of multiple distributions by respecting the underlying geometry of the data. It effectively combines multimodal distributions and aligns them before averaging, preserving distributional characteristics. Specifically, note that a GP is equivalent to a probability measure defined on $\mathcal{X}$. The central model, represented by the Wasserstein barycenter, is defined as the probability measure that minimizes the squared Wasserstein distance to local GP models:
\begin{equation}
\label{eq.c1}
    f^c = \inf_{f'\in\mathcal{P}\left(\mathcal{X}\right)}\sum_{n=1}^{N}\left[W_2\left(f',\tilde{f}_{n}\right)\right]^2.
\end{equation}
Here, $\mathcal{P}\left(\mathcal{X}\right)$ denotes the set of all probability measures on $\mathcal{X}$, and $W_2\left(\,\cdot\,,\,\cdot\,\right)$ denotes the 2-Wasserstein distance:
\begin{equation*}
W_2(\mu, \nu)\!\doteq\!\left(\inf _{\gamma \in \Gamma[\mu, \nu]} \int_{\left(x,x'\right)\in \mathcal{X} \times \mathcal{X}}\! \!\left\|x- x'\right\|^2\, \mathrm{d} \gamma\left(x, x'\right)\!\right)^{\frac{1}{2}},
\end{equation*}
where $\mu$ and $\nu$ are probability measures defined on $\mathcal{X}$, and $\Gamma[\mu, \nu]$ denotes the set of probability measures defined on $\mathcal{X} \times \mathcal{X}$, with marginal distributions $\mu$ and $\nu$; see \cite{masarotto2019procrustes}. The central model $f^c$ defined by a Wasserstein barycenter of multiple GP models is itself a GP.
\begin{proposition}[\cite{mallasto2017learning}]\label{prop1}
    Let $\left\{\tilde{f}_n\right\}_{i=1}^N$ be a set of GPs with $\tilde{f}_n \sim \mathcal{G P}\left(\tilde{\mu}_n\left(x\right),\tilde{K}_n\left(x,x'\right)\right)$. There exists a unique barycenter $f^c \sim \mathcal{GP} \left(\mu^c(x), K^c\left(x,x'\right)\right)$ defined as in (\ref{eq.c1}). If $f^c$ is non-degenerate, the associated mean function $\mu^c(x)$ and the kernel function $K^c\left(x,x'\right)$ satisfy that
$$
\mu^c(x)=\frac{1}{N}\sum_{n=1}^N \tilde{\mu}_n\left(x\right),$$
and 
$$\sum_{n=1}^N \left(\Phi_{K^c}^{\frac{1}{2}} \Phi_{\tilde{K}_n} \Phi_{K^c}^{\frac{1}{2}}\right)^{\frac{1}{2}}=N\Phi_{K^c},
$$
where $\Phi_{K}$ denotes the operator that is associated with the kernel function $K\left(x,x'\right)$.
\end{proposition}
Proposition \ref{prop1} defines the Wasserstein barycenter of GPs with the help of kernel operators. In practice, we compute the barycenter by discretizing GPs to multivariate normal distributions \cite{masarotto2019procrustes};  {we elaborate on this in}  Section \ref{sec.imp}.

We now introduce a general acquisition function for our framework, given by:
\begin{equation}
\label{eq.a}
\alpha(\mathbf{x}) = \underbrace{\alpha^c(\mathbf{x})}_{\text{central}} + \beta_t \sum_{n=1}^{N} \underbrace{\alpha_n(x_n)}_{\text{local}},
\end{equation}
where $\mathbf{x} = (x_1, x_2, \ldots, x_N)$ represents the joint vector of decision variables for the agents. The function $\alpha^c(\mathbf{x})$ denotes the acquisition function based on the central model, which selects $N$ decision variables jointly, while $\alpha_n(x_n)$ refers to the acquisition function for the $n$-th local model, focusing specifically on $x_n$. This approach facilitates a collaborative acquisition strategy that leverages the central model while maintaining differentiation through consideration of the local models. The hyperparameter $\beta_t$ is used to balance this trade-off in the $t$-th iteration, indicating the extent of collaboration. When $\beta_t\rightarrow 0$, the acquisition function approaches a parallel acquisition function focusing on the central model, eliminating the effects of the differences between local models. When $\beta_t\rightarrow \infty$, each decision variable $x_n$ is selected to maximize its acquisition function based on the local model, i.e., there is no collaboration between agents. 

Specifically, we consider an increasing sequence of $\beta_t$ and let $\beta_t \rightarrow \infty$. In early iterations, there is insufficient data, so the central model collects information from local models to help construct a more accurate approximation of the objective function, which facilitates the optimization procedure \cite{yue2023collaborative}. On the other hand, constructing a central GP model requires approximation (see details in the next section), which introduces additional bias and uncertainty. As the iterations progress and more data is collected, each local model becomes more reliable on its own, reducing the reliance on the central model. Therefore, we attach more weight to each local GP with increasing $\beta_t$. We also discuss the effects of this hyperparameter through numerical experiments.

Furthermore, this collaborative acquisition function is flexible in terms of the selections of $\alpha^c(\mathbf{x})$ and $\alpha_n\left(x\right)$ in different scenarios. For example, $\alpha^c(\mathbf{x})$ and $\alpha_n(x)$ can be parallel knowledge gradient ($q$-KG) \cite{wu2016parallel} and knowledge gradient (KG), parallel expected improvement \cite{wang2020parallel} and expected improvement, and parallel predictive entropy search \cite{shah2015parallel} and entropy search respectively. We compare different selections of collaborative acquisition functions in Section \ref{sec.exp1}.

We focus on the KG function due to its proven success in practical implementations. We propose an acquisition function named \textit{collaborative knowledge gradient} (Co-KG):
\begin{equation}
\label{eq.cokg}
\begin{aligned}
    \alpha_{\text{Co-KG}}\left(\mathbf{x}\right) \doteq&\mathbb{E}_{\tilde{\mathcal{S}}}\left\{\max_{x'\in\mathcal{X}}\mathbb{E}\left[f^c\left(x'\right)\mid \tilde{\mathcal{S}}^*_{\mathbf{x}}\right]\right.\\&\left.+\beta_t \left\{\sum_{n=1}^N\max_{x'\in\mathcal{X}}
    \mathbb{E}\left[f\left(x'\right)\mid\tilde{\mathcal{S}}^*_{x_n}\right]\right\}\right\}.\end{aligned}
\end{equation}
Here $\tilde{\mathcal{S}} = \cup_{n=1}^N \tilde {\mathcal{S}}_n$ represents the historical observations, $\tilde{\mathcal{S}}_\mathbf{x}^* = \cup_{n=1}^N \tilde {\mathcal{S}}^*_{x_n}$ represents the updated dataset if additional observations are collected as $\tilde{\mathcal{S}}_{x_n}^* = \tilde{\mathcal{S}}_n\cup\left\{\left(x_n,y\left(x_n\right)\right)\right\}$, consistent with the definition in (\ref{eq.kg}). In other words, $\alpha^c\left(\mathbf{x}\right)$ is selected to be $q$-KG defined on the central GP model 
$\alpha _{\text{$q$-KG} }\left ( \mathbf{x} \right ) = \mathbb{E}_{\tilde{\mathcal{S}}}\left [ \max_{x'\in\mathcal{X}} \mathbb{E}\left [ f^c\left(x'\right)\mid \tilde{\mathcal{S}} ^*_{\mathbf{x}}\right ] \right ],$
which selects $N$ decision variables in parallel based on the central GP model $f^c$. Additionally, $\alpha_n\left(x\right)$ is a regular KG function for each local GP, as defined in (\ref{eq.kg}). Although (\ref{eq.cokg}) involves the set of observations, the construction of the function is facilitated by the central and local GP models, rather than directly utilizing the data, i.e., the data privacy is preserved.

With the acquisition function defined in (\ref{eq.cokg}), our framework for collaborative BO  {(formalized in Algorithm~\ref{alg:1})} proceeds as follows: First, each agent independently and randomly collect observations and construct GP models in a warm-up stage  {(line 2)}. Then, the procedure begins to iterate. At each iteration, the server first collects the GP models from the agents  {(line 4)} to construct a central GP model  {(line 5)}. Next, the server maximizes the acquisition function (\ref{eq.cokg}) to select the decision variables $\mathbf{x} = \left(x_1, x_2, \ldots, x_N\right)$ for each agent  {(line 6)}. Consequently, each agent collects observations with the decision variable selected by the server, updates the local GP model, and the procedure moves to the next iteration. This procedure is repeated for $T$ iterations. When the last iteration terminates, each agent submits the optimizer $\hat{x}^*_n$ and the corresponding optimal value $\mu^*_n$ of $\max_{x\in\mathcal{X}}\tilde{\mu}_{n}(x)$ to the server  {(line 10)}. The server then determines $\hat{x}^*$ as the decision variable that maximizes across all agents' optimal values  {(line 11)}. Implementation  details are given in  the next section. Here we focus on the general procedure and provide  consistency results. We make the following assumptions.
\begin{assumption}
\label{assumption1}
    \begin{enumerate}\item[]
        \item The feasible set $\mathcal{X}$ is a compact set.
        \item Given the kernel function $K\left(x,x'\right)$, there exists a constant $\tau>0$ and a continuous function $\rho:\mathbb{R}^d \mapsto \mathbb{R}_{+}$ such that $K\left(x,x'\right) = \tau^2 \rho\left(x-x'\right)$. Moreover,  {$\rho$ satisfies:}
        \begin{enumerate}
            \item $\rho(|\delta|)=\rho(\delta)$, where $|\,\cdot\,|$  {is interpreted component-wise;}
            \item $\rho(\delta)$ is decreasing in $\delta$ component-wise for $\delta \geqslant \mathbf{0}$;
\item $\rho(\mathbf{0})=1, \rho(\delta) \rightarrow 0$ as $\|\delta\| \rightarrow \infty$%, where $\|\cdot\|$ denotes the Euclidean norm;
\item there exist some $0<C<\infty$ and $\varepsilon, u>$ 0 such that
$$
1-\rho(\delta) \leqslant \frac{C}{|\log (\|\delta\|)|^{1+\varepsilon}},
$$
for all $\delta$ such that $\|\delta\|<u$.
           
        \end{enumerate}
    \end{enumerate}
\end{assumption}
The second condition in the assumptions above is a standard requirement for general kernel functions, such as the RBF kernel function in (\ref{eq.rbf}). For other kernel functions that satisfy this condition, we refer to \cite{ding2022knowledge}.
\begin{theorem}
\label{thm.1}
    Under Assumption \ref{assumption1}, the collaborative BO with Co-KG summarized in Algorithm \ref{alg:1} is consistent. That is, 
    \begin{equation*}
        \lim_{T\rightarrow\infty } f\left(\hat{x}^*\right)\stackrel{a.s.}{=}\max_{x\in\mathcal{X}}f(x).
    \end{equation*}
\end{theorem}
Specifically, as the iterations \( T \rightarrow \infty \), the posterior variance of each local GP model shrinks to zero everywhere, and the posterior mean function converges to the true objective function. Consequently, the limit of function value at the selected decision variable approaches the ground-truth optimal value. The proof is deferred to the supplements \footnote{The online version of the manuscript is in \url{https://github.com/jd-anderson/Collab_Bayesian_Opt/}} in the online version.

\begin{algorithm}[ht!]
\caption{Collaborative BO with Co-KG.}
\label{alg:1}
\begin{algorithmic}[1]
\STATE {\bfseries Input:} 
The prior kernel function of local GP models;
\STATE Warm-up stage: Each agent collects observations and updates GP models as in (\ref{eq.gp});

\FOR{$t=1,2,\ldots,T$}

\STATE The server collects local GP models from each agent;

\STATE The server constructs the central GP as in (\ref{eq.c1});

\STATE The server selects $\mathbf{x} = \left(x_1,x_2,\ldots,x_N\right)$ by maximizing the Co-KG function (\ref{eq.cokg});
\STATE Each agent collects a new observation at $x_n$;
\STATE Each agent updates the posterior mean $\tilde{\mu}_{n}(x)$ and the posterior kernel function $\tilde{K}_{n}\left(x,x'\right)$;

\ENDFOR
\STATE Each agent reports $\hat{x}_n^* = \arg\max_{x\in\mathcal{X}}\tilde{\mu}_{n}(x)$ and $\mu^*_n = \tilde{\mu}_{n}\left(\hat{x}_n^*\right)$ to the server;
\STATE The server outputs $\hat{x}^* = \arg\max_{x_1^*,x_2^*,\ldots,x_N^*} \mu^*_n$.

\end{algorithmic}
\end{algorithm}

\section{Implementation}\label{sec.imp}

Here we describe the process of calculating the Wasserstein barycenter of a Gaussian process and the optimization of Co-KG based on discretization, which is a standard procedure in BO literature. We discretize the feasible set $\mathcal{X}$ (the domain of GP) to $\mathcal{X}_D = \left\{x^{(1)},x^{(2)},\ldots,x^{(D)}\right\}$, where $\left|\mathcal{X}_D\right| = D$. The agents send both $\tilde{\bm{\mu}}_{n} = \left ( \tilde{\mu}_{n}\left ( x^{(1)} \right ),\tilde{\mu}_{n}\left ( x^{(2)}\right),\ldots, \tilde{\mu}_{n}\left ( x^{(D)} \right)\right )^{\top}\in \mathbb{R}^{D} $ and $$\tilde{\bm{K}}_n\!\! =\!\! \begin{pmatrix}
  \tilde{K}_n\left ( x^{(1)},x^{(1)} \right ) & \!\dots\!  & \tilde{K}_n\left ( x^{(1)},x^{(D)} \right ) \\
 \vdots & \! \ddots\!& \vdots\\
 \tilde{K}_n\left ( x^{(D)},x^{(1)} \right )  &\!\dots\!  &\tilde{K}_n\left ( x^{(D)},x^{(D)} \right ) 
\end{pmatrix}\!\in\!\mathbb{R}^{D\times D}$$
to the server, where $\tilde{\mu}_n(x)$ and $\tilde{K}_n\left(x,x'\right)$ denote the posterior mean and kernel functions associated with the $n$-th local model, updated as in (\ref{eq.gp}). We note that, in the $t$-th iteration, the posterior mean and kernel functions are constructed using more than $t$ observations since agents have independently collected observations during the warm-stage as in Algorithm \ref{alg:1}. Then, we attain the mean vector of the (discretized) central model as $\bm{\mu}^c = \frac{1}{N} \sum_{n=1}^{N}\tilde{\bm{\mu}}_n$ and the kernel matrix $\bm{K}^c$ by solving the equation
\begin{equation}
\label{eq.matrixequation}
    \sum_{n=1}^N \left({\left ( \bm{K}^c \right ) }^{\frac{1}{2}} {\tilde{\bm{K}}_n} {\left ( \bm{K}^c \right ) }^{\frac{1}{2}}\right)^{\frac{1}{2}}=N{\bm{K}^c}.
\end{equation}
The equation (\ref{eq.matrixequation}) can be efficiently solved by numerical methods. When the size of discretization $D\rightarrow\infty$, $\mathcal{N}\left (\bm{\mu}^c, \bm{K}^c \right ) $ approximates $f^c$ defined in (\ref{eq.c1}) arbitrarily well \cite{mallasto2017learning}.

This discretization also helps maintain data privacy. Specifically, given the full knowledge of (i) the posterior mean and covariance functions, and (ii) all the decision variables that have been selected, the data can be inferred as in (\ref{eq.gp}). In our implementation, since the posterior mean and covariance functions that the agents send are discretized, this inference cannot be implemented. Additionally, as in Algorithm \ref{alg:1}, there is a warm-up stage where the selected decision variables by each agent are not revealed to the server. These two steps main the data privacy requirement, i.e., the server is not able to reveal the exact values of data from the posterior mean and covariance functions sent by the agents.

We now describe the optimization of the acquisition function Co-KG defined in (\ref{eq.cokg}). Since both $q$-KG and KG involve taking expectations, and do not  {admit} an explicit  {solution},  we use a Monte Carlo (MC) based approximation. Specifically, we express the posterior mean function of the central GP model $\mathbb{E}\left [ f^c\left(x'\right)\mid \tilde{\mathcal{S}}^*_{\mathbf{x}} \right ]$ as 
\begin{equation*}
    \mu^c\left(x'\right)+ \bm{K}^c\left(\mathbf{x},x'\right) ^{\top}\left(\mathbf{\Sigma}\left(\mathbf{x}\right)\right)^{-1}\left ( y^*\left(\mathbf{x}\right)-\bm{\mu}^c\left ( \mathbf{x} \right )  \right ),\end{equation*}
where $\bm{K}^c\left(\mathbf{x},x'\right)\! =\!\left ( K^c\left ( x_1,x' \right ),\ldots,K^c\!\left ( x_N,x'\right )  \!\right ) ^{\top}\in\mathbb{R}^N $, 
$$\mathbf{\Sigma}\!\left(\mathbf{x}\right)\!=\!\!\begin{pmatrix}
K^c\left ( x_1,x_1 \right )   &\!\!\! \dots \!\!\!&K^c\left ( x_1,x_N \right ) \\
 \vdots & \!\!\!\ddots\!\!\! & \vdots\\
 K^c\left ( x_N,x_1 \right ) &\!\!\! \dots \!\!\!& K^c\left ( x_N,x_N \right )
\end{pmatrix}\!+\sigma^2_\epsilon \bm{I}_N\!\in\! \mathbb{R}^{N\times N},$$
and $\bm{\mu}^c\left ( \mathbf{x} \right ) =\left ( \mu^c\left ( x_1 \right ) , \mu^c\left ( x_2 \right ),
\ldots ,  \mu^c\left ( x_N \right )\right ) ^{\top}\in\mathbb{R}^{N}$. Note that $y^*\left(\mathbf{x}\right)-\bm{\mu}^c\left ( \mathbf{x}\right)\sim \mathcal{N}\left(\bm{0}, \mathbf{\Sigma}\left(\mathbf{x}\right)\right).$ We have 
\begin{equation*}
\label{eq.expression}
    \mathbb{E}\left [ f^c\left(x'\right)\mid \tilde{\mathcal{S}}^*_{\mathbf{x}} \right ] = \mu^c\left(x'\right)+\bm{\sigma}^c\left(\mathbf{x},x'\right)\bm{\xi},
\end{equation*}
where $\bm{\sigma}^c\left(\mathbf{x},x'\right) = \bm{K}^c\left(\mathbf{x},x'\right) ^{\top}\left(\bm{D}\left(\mathbf{x}\right)^{\top}\right)^{-1}$, $\bm{D}\left(\mathbf{x}\right)$ is the Cholesky factor of the matrix $\mathbf{\Sigma}\left(\mathbf{x}\right)$, and $\bm{\xi}\sim\mathcal{N}\left(\bm{0},\bm{I}_{N}\right)$. In this manner, when a sample of $\bm{\xi}$ and the decision variables $\mathbf{x}$ are fixed, $\mathbb{E}\left [ f^c\left(x'\right)\mid \tilde{\mathcal{S}}^*_{\mathbf{x}}\right ]$ is a deterministic function of $x'$ so can be optimized over $x'\in \mathcal{X}$. With a similar argument and $\xi\sim\mathcal{N}\left(0,1\right)$, the posterior mean function associated with the $n$-th local GP model can be expressed as 
\begin{equation*}
    \mathbb{E}\left[f\left(x'\right)\mid \tilde{\mathcal{S}}^*_{x_n}\right] = \tilde{\mu}_{n}\left(x'\right)+\tilde{\sigma}_{n}\left ( x_n,x' \right ) \xi,
\end{equation*}
where $\tilde{\sigma}_{n}\left(x_n,x'\right)=\tilde{K}_n\left(x_n,x'\right)/\sqrt{\tilde{K}_n\left(x_n,x_n\right)+\sigma^2_\epsilon}.$

Therefore, to approximate the Co-KG function in (\ref{eq.cokg}), we first generate $\left \{ \bm{\xi}_1,\bm{\xi}_2,\ldots,\bm{\xi}_M \right \}$, where $\bm{\xi}_m \stackrel{i.i.d.}{\sim} \mathcal{N}\left(\bm{0},\bm{I}_{N}\right)$. Then, we have the MC approximation 
\begin{equation}
\begin{aligned}
    \label{eq.ackg}
    \hat{\alpha }_{\text{Co-KG}}\left(\mathbf{x}\right) \doteq&\max_{\mathbf{z}}\frac{1}{M}\sum_{m=1}^{M} \left\{\left \{\mu^c\left(z_m\right)+\bm{\sigma}^c\left(\mathbf{x},z_m\right)\bm{\xi}_m \right \}\right.\\&\left.+\beta_t\sum_{n=1}^N \left\{\tilde{\mu}_{n}\left(z_{n,m}\right)+\tilde{\sigma}_{n}\left ( x_n,z_{n,m} \right ) \xi_{m,n}\right\}
\right\}.\end{aligned}
\end{equation}
Here $\mathbf{x} =\left ( x_1,x_2,\ldots, x_N \right ) $ is the set of decision variables; $\mathbf{z} =\left ( z_1,z_2,\ldots, z_M,z_{1,1},\ldots,z_{N,M}\right )$ are optimizers to maximize the sampled posterior mean functions; and $x_n,z_m,z_{n,m}\in \mathcal{X}_D$. In addition, $\xi_{m,n}$ denotes the $n$-th entry of $\bm{\xi}_m$. Since $x_n,z_m$ and $z_{n,m}$ are restricted to be within $\mathcal{X}_D$, we have the exact values of quantities in (\ref{eq.ackg}) in terms of $\tilde {\bm\mu}_n$'s, $\tilde {\bm K}_n$'s, $\bm{\mu}^c$, and $\bm {K}^c$. Thus, we can optimize $\mathbf{x}$ and $\mathbf{z}$ altogether to maximize the approximated Co-KG function; see also \cite{balandat2020botorch}. %Moreover, since the de

Let $\alpha^* = \max_{\mathbf{x}}\alpha_{\text{Co-KG}}\left ( \mathbf{x} \right )$ and $\hat{\alpha}^* = \max_{\mathbf{x}}\hat{\alpha}_{\text{Co-KG}}\left ( \mathbf{x} \right )$. Also, we let $\hat{\mathbf{x}}^* \in\arg\max_{\mathbf{x}}\hat{\alpha}_{\text{Co-KG}}\left ( \mathbf{x} \right )$ and $\bm{\mathcal{X}}^*$ be the set of optimizers of $\max_{\mathbf{x}}\alpha_{\text{Co-KG}}\left ( \mathbf{x} \right )$. Next we provide the consistency of the maximization based on this MC approximation. 
\begin{theorem}
\label{thm.2}
   If the kernel function is continuously differentiable, we have
    \begin{equation*}
        \lim_{M\rightarrow\infty}\hat{\alpha}^* \stackrel{a.s.}{=}\alpha^* \quad
        \text{and}\quad
        \lim_{M\rightarrow\infty}\operatorname{dist}\left(\hat{\mathbf{x}}^*, \bm{\mathcal{X}}^*\right)\stackrel{a.s.}{=} 0 ,
    \end{equation*}
    where $M$ denotes the number of samples generated to construct the MC approximation. Furthermore, $\forall \delta>0, \exists K<\infty, \beta_t>0$ such that 
    \begin{equation*}
        \mathbb{P}\left(\operatorname{dist}\left(\hat{\mathbf{x}}_N^*, \mathcal{X}^*\right)>\delta\right) \leqslant K e^{-\beta_t M}\quad \forall M \geqslant 1.
    \end{equation*}
\end{theorem}
{\begin{proof}
    We omit the details due to lack of space, however it follows from a simple adaptation of Theorem \ref{thm.2} from \cite{balandat2020botorch}, which is itself grounded in the theoretical foundations of the sample average approximation method \cite{kleywegt2002sample}.
\end{proof}}
{We note that the condition that the kernel function is continuously differentiable guarantees that the objective function (the sample path of the GP) is continuously differentiable as well. As a result, the sample path of the central GP is also continuously differentiable, supported by the properties of the Wasserstein distance and the Fréchet mean \cite{rao1987differential,panaretos2019statistical}}.

\section{Experiments}

We first compare the performances of different collaborative acquisition functions in order to justify the use of the knowledge gradient function. Next, we explore the algorithm performance on the hyperparameter selection. Lastly, we compare our approach (Algorithm~\ref{alg:1}) with several baseline approaches, as well as a parallel BO algorithm without data privacy concerns.

Across all sets of experiments, we fix the number of agents at $N=4$. We also normalize the feasible set to the box $\mathcal{X}=[0, 1]^2$ and discretize the feasible set using a $20 \times 20$ uniform meshgrid. Additional experiments on the effects of the number of agents and different discretization strategies are also included in the supplementary materials. Our experiments were conducted with Botorch \cite{balandat2020botorch} and Python 3.9 on a computer equipped with two AMD Ryzen Threadripper 3970X 32-Core Processors, 128 GB memory, and a Nvidia GeForce RTX A6000 GPU with 48GB of RAM.

\begin{figure}[ht!]
    \centering
    \begin{minipage}[b]{0.24\textwidth}
        \centering
        \includegraphics[width=\linewidth]{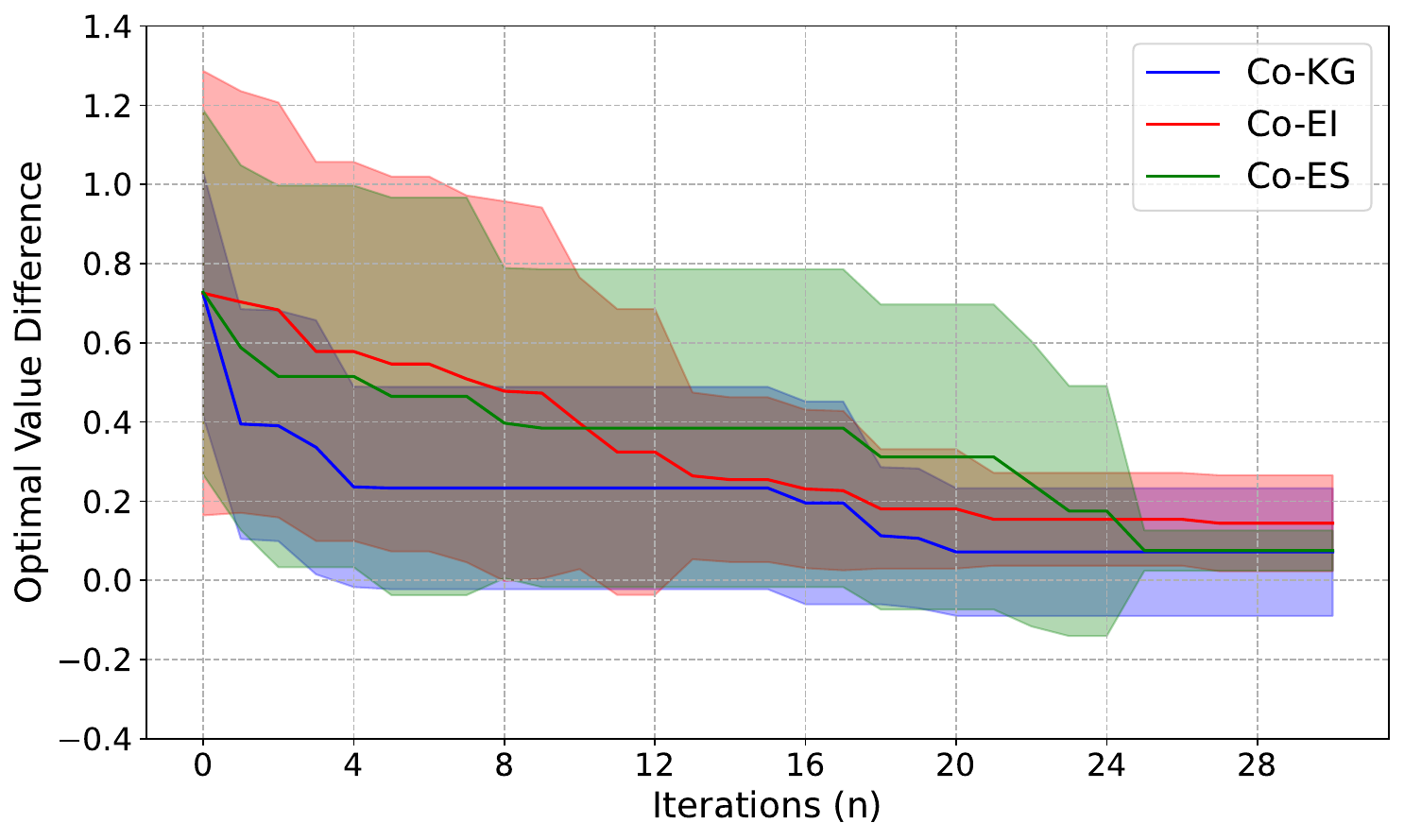}
        \caption{ Optimal value differences with different acquisition functions on the black-box objective function $f_1(x)$.}
        \label{fig:sub1}
    \end{minipage}
    \hfill
    \begin{minipage}[b]{0.24\textwidth}
        \centering
        \includegraphics[width=\linewidth]{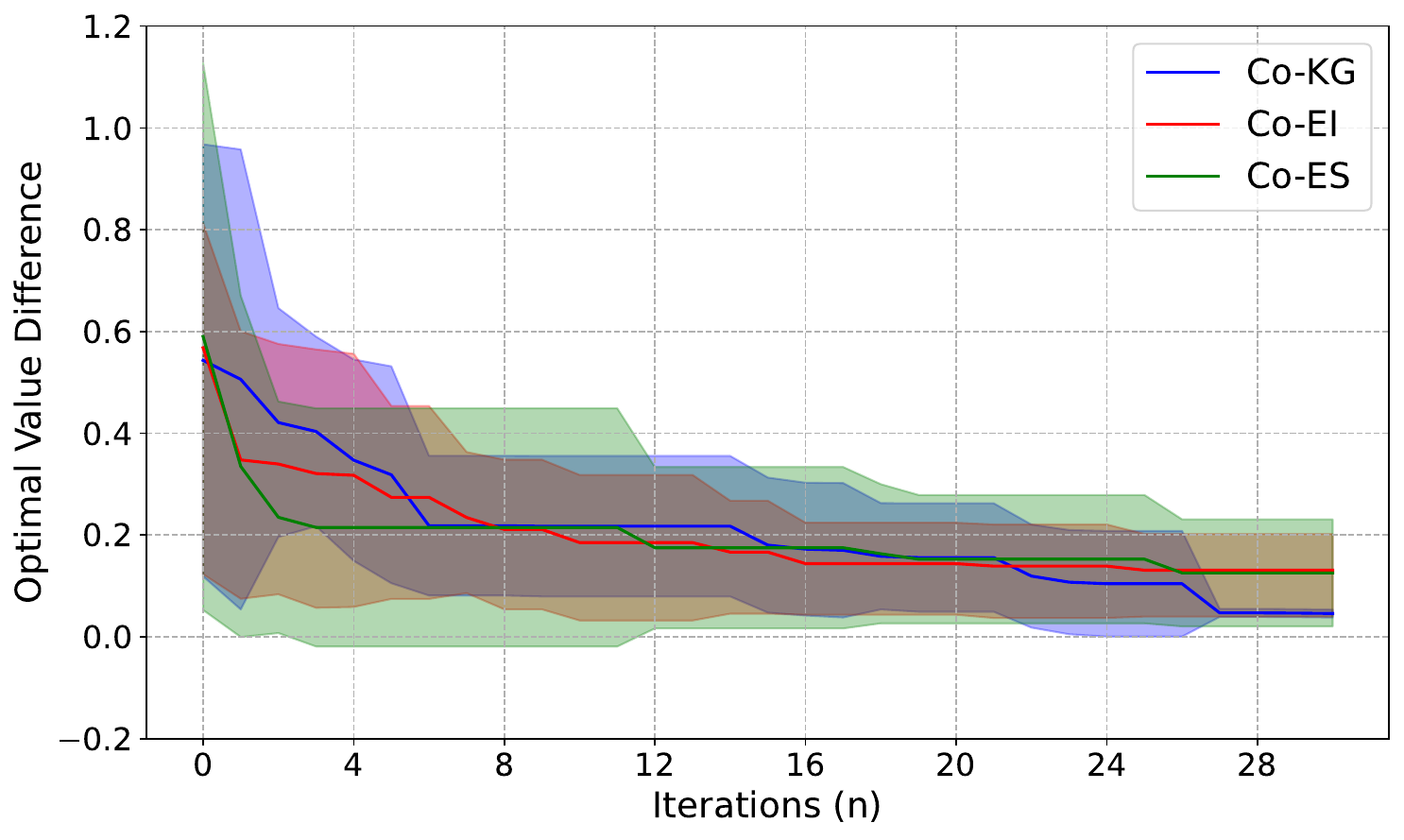}
        \caption{ Optimal value differences with different acquisition functions on the black-box objective function $f_2(x)$.}
        \label{fig:sub2}
    \end{minipage}
    
    %\caption{Overall caption for both images.}
    \label{fig:func}
\end{figure}

\subsection{Collaborative Acquisition Function Comparison}
\label{sec.exp1}
We conduct experiments to compare different collaborative acquisition functions within the general form of (\ref{eq.a}). Specifically, we consider (i) the collaborative knowledge gradient (Co-KG) function as in (\ref{eq.cokg});
    (ii) the collaborative expected improvement (Co-EI) function with $\alpha^c\left(\mathbf{x}\right)$ selected as the parallel expected improvement function and $\alpha_n(x)$ selected as the expected improvement function \cite{wang2020parallel}; and
    (iii) the collaborative entropy search (Co-ES) function with $\alpha^c\left(\mathbf{x}\right)$ selected as the parallel expected improvement function and $\alpha_n(x)$ selected as the expected improvement function \cite{shah2015parallel}.

Regarding the hyperparameter in the collaborative functions, we set $\beta_t = \log\left ( 2t+1 \right ) $. We compare the performance of these algorithms using the two functions: a function with quadratic terms and trigonometric terms
\begin{equation*}
    \label{eq.f1} f_1(x)  = x_1^2 + x_2^2 + \sin(2 \pi x_1) + \cos(2 \pi x_2),
\end{equation*}
and the Rosenbrock function %\cite{picheny2013benchmark}
\begin{equation*}
    \label{eq.f2} f_2(x) = (1 - x_1)^2 + 100(x_2 - x_1^2)^2.
\end{equation*}
 When collecting the observations, we add Gaussian noise as $y = f(x)+\epsilon$, where $\epsilon\sim \mathcal{N}(0,\sigma^2_\epsilon=0.02)$.  Regarding the variance of the noise, $\sigma_\epsilon^2$, each agent estimates it from observations collected in the warm-up stage using maximum likelihood estimation, denoted by $\hat{\sigma}^2_n$. Then the variance is $\left ( \sum_{n=1}^N \hat{\sigma}^2_n \right ) /N  $ fixed by the server. The experimental results are in Figure \ref{fig:sub1} and Figure \ref{fig:sub2}, where we report 
the \textit{optimal value difference} defined by 
\begin{equation*}
    \arg\max_{x\in\mathcal{X}}f(x) - f\left(\hat{x}^*\right).
\end{equation*}
Here $\hat{x}^*$ is the maximizer selected by the server after the iterations end; see Algorithm \ref{alg:1}. The experimental results presented are mean performances based on 10 repetitions, with standard deviations represented by a shadow around the mean-value line. Each repetition of the optimization procedure includes 30 iterations, and each agent has 5 observations associated with randomly selected decision variables before the start of the iterations. From the experimental results in Figure \ref{fig:sub1}, we observe that Co-KG consistently outperforms Co-EI and Co-ES in terms of achieving lower optimal value differences across iterations. In contrast, in Figure \ref{fig:sub2}, Co-KG does not perform as well as Co-EI or Co-ES in the initial iterations but surpasses them as the iterations increase. The reason is that the Co-KG function is overconfident in the set of experiments associated with $f_2(x)$, relying too heavily on the conditional mean function. When there is insufficient data, the conditional mean is not accurate enough in approximating the objective function. However, as the iterations increase and more data becomes available, Co-KG, which relies on the conditional mean, outperforms the other two approaches due to a more accurate approximation of the objective function. This advantage is also evident in the results shown in Figure \ref{fig:sub1}. Since $f_1(x)$ is relatively simple, the conditional mean serves as a satisfactory approximation even when there is insufficient data at the beginning of the iterations, leading to Co-KG's preferable performance. Considering both experimental results, we focus on the acquisition function Co-KG in the following experiments.

\subsection{Hyperparameter Analysis}

We now explore the dependence of Co-KG on the hyperparameter $\beta_{t}$. We numerically evaluate how a time-varying hyperparameter $\beta_t$ will affect the performance of Co-KG, where $t$ denotes the iteration of the procedure. Specifically, we consider a decreasing sequence of hyperparameters $\beta_{t} = e^{-t/2}$ and an increasing sequence $\beta_{t} = \log(2t+1)$. We also include $\beta_t = 1$ for comparison. We follow the setting of section~\ref{sec.exp1} to further investigate the impact of hyperparameter.

\begin{figure}[ht!]
    \centering
    \begin{minipage}[b]{0.24\textwidth}
        \centering
        \includegraphics[width=\linewidth]{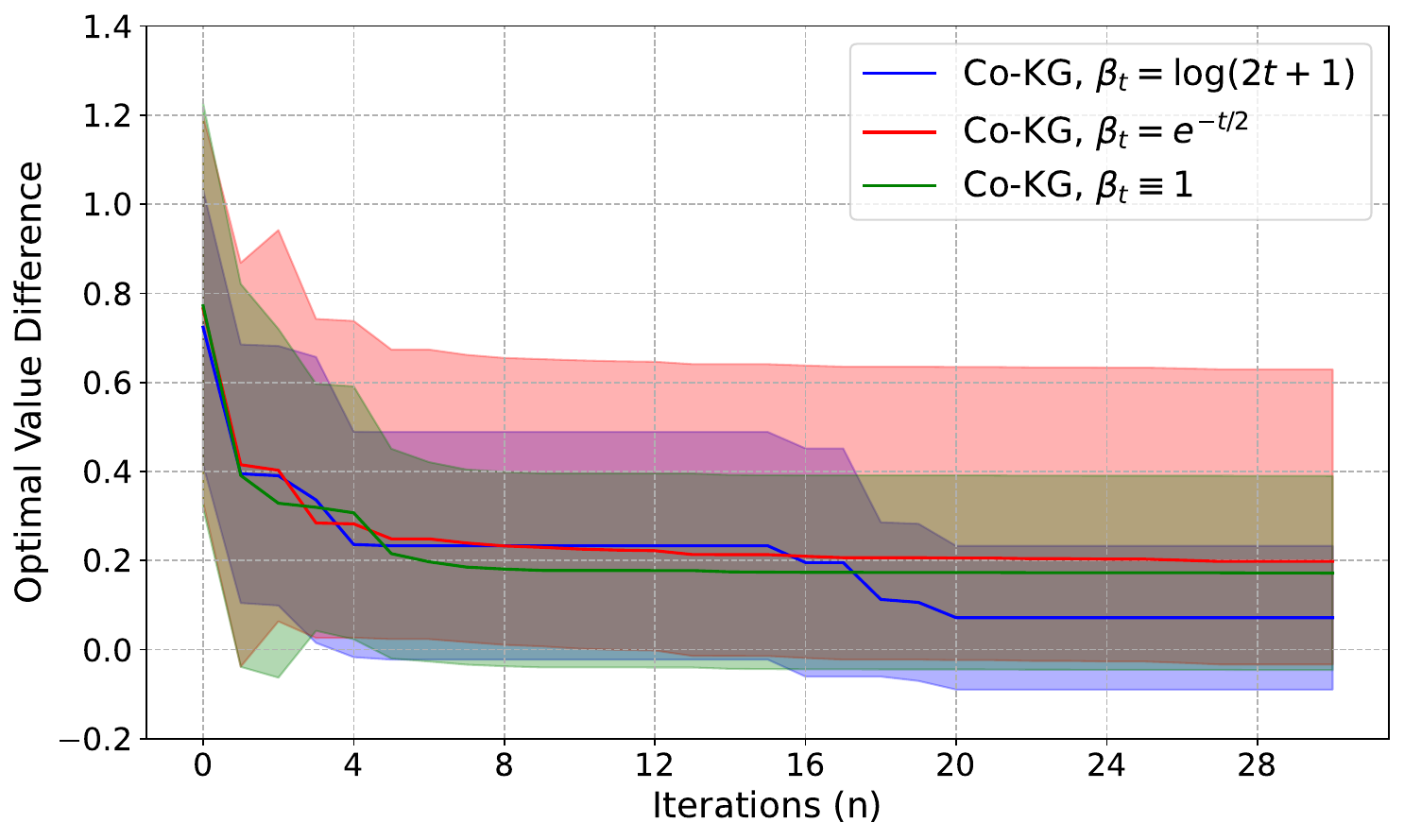}
        \caption{Optimal value difference in iterations with different selections of $\beta_t$ on $f_{1}(x)$.}
        \label{fig:sub11}
    \end{minipage}
    \hfill
    \begin{minipage}[b]{0.24\textwidth}
        \centering
        \includegraphics[width=\linewidth]{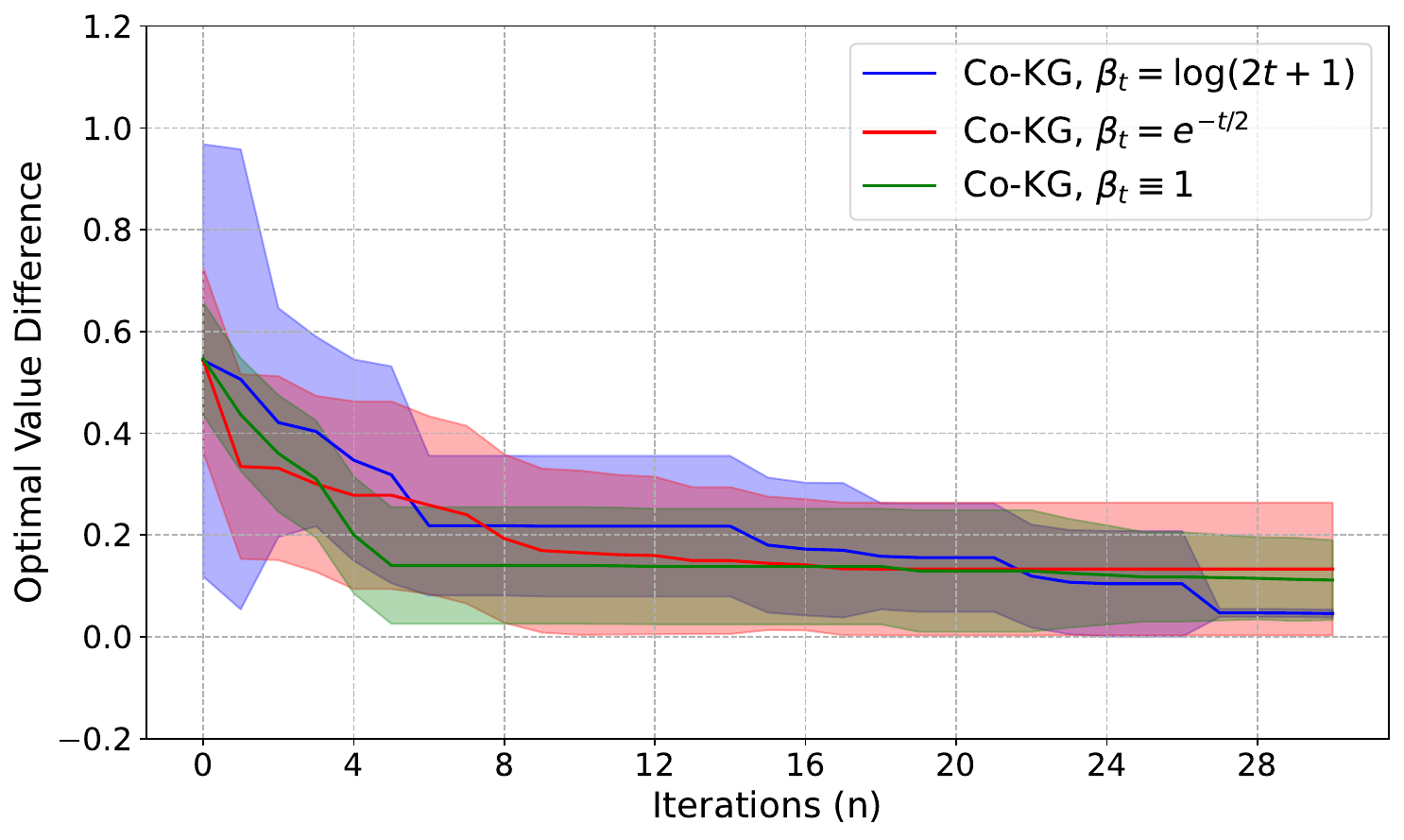}
        \caption{Optimal value difference in iterations with different selections of $\beta_t$ on $f_{2}(x)$.}
        \label{fig:sub21}
    \end{minipage}
    
    %\caption{Overall caption for both images.}
    \label{fig:hyper}
\end{figure}

The experimental results are included in {Figure~\ref{fig:hyper}}. We see that the performance of Co-KG does not significantly depend on the hyperparameter $\beta_t$ in the initial iterations. In comparison, as the iterations progress, an increasing sequence of $\beta_t$ outperforms the other two selections in both sets of experiments. The reason is two-fold:  In early iterations, there is insufficient data, so the central model collects information from local models to help construct a more accurate approximation of the objective function, which helps the overall optimization procedure \cite{yue2023collaborative}. On the other hand, constructing a central GP model requires approximation, which introduces additional bias and uncertainty. As the iteration increases and more data is collected, each local model becomes more reliable, reducing the reliance on the central model. Additionally, from the perspective of the surrogate models, $\beta_t$ also addresses the exploration-exploitation trade-off. A higher $\beta_t$ favors the exploration of local models to gather diverse information about the objective function. Increasing  $\beta_t$ as the iterations progress, to manage the exploration-exploitation trade-off, is also supported by classical BO literature \cite{srinivas2010gaussian}.

% \end{figure}
\begin{figure}[ht!]
    \centering
    \begin{minipage}[b]{0.24\textwidth}
        \centering
        \includegraphics[width=\linewidth]{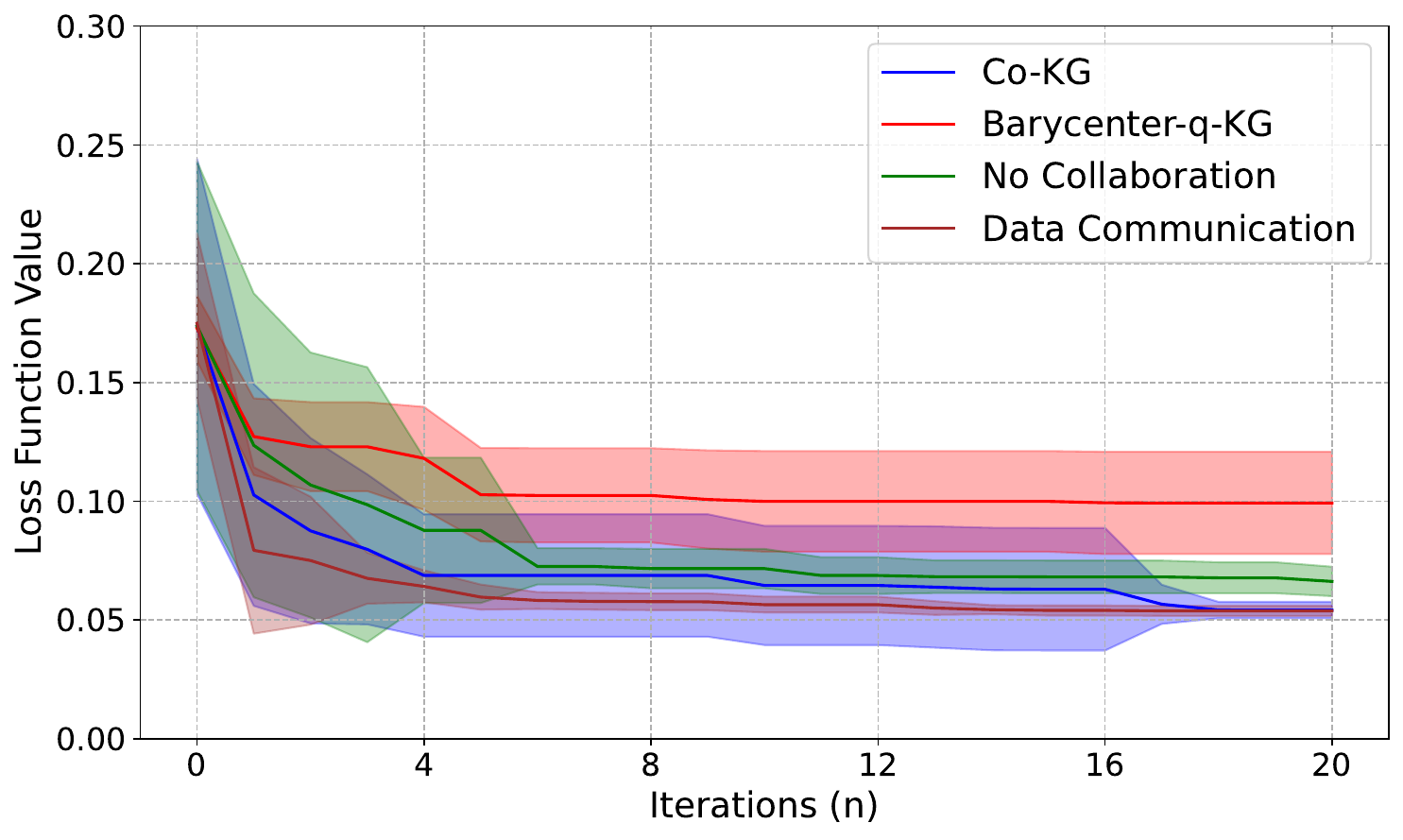}
        \caption{Loss function values in iterations with compared BO approaches on Breast Cancer Dataset.}
        \label{fig:sub3}
    \end{minipage}
    \hfill
    \begin{minipage}[b]{0.24\textwidth}
        \centering
        \includegraphics[width=\linewidth]{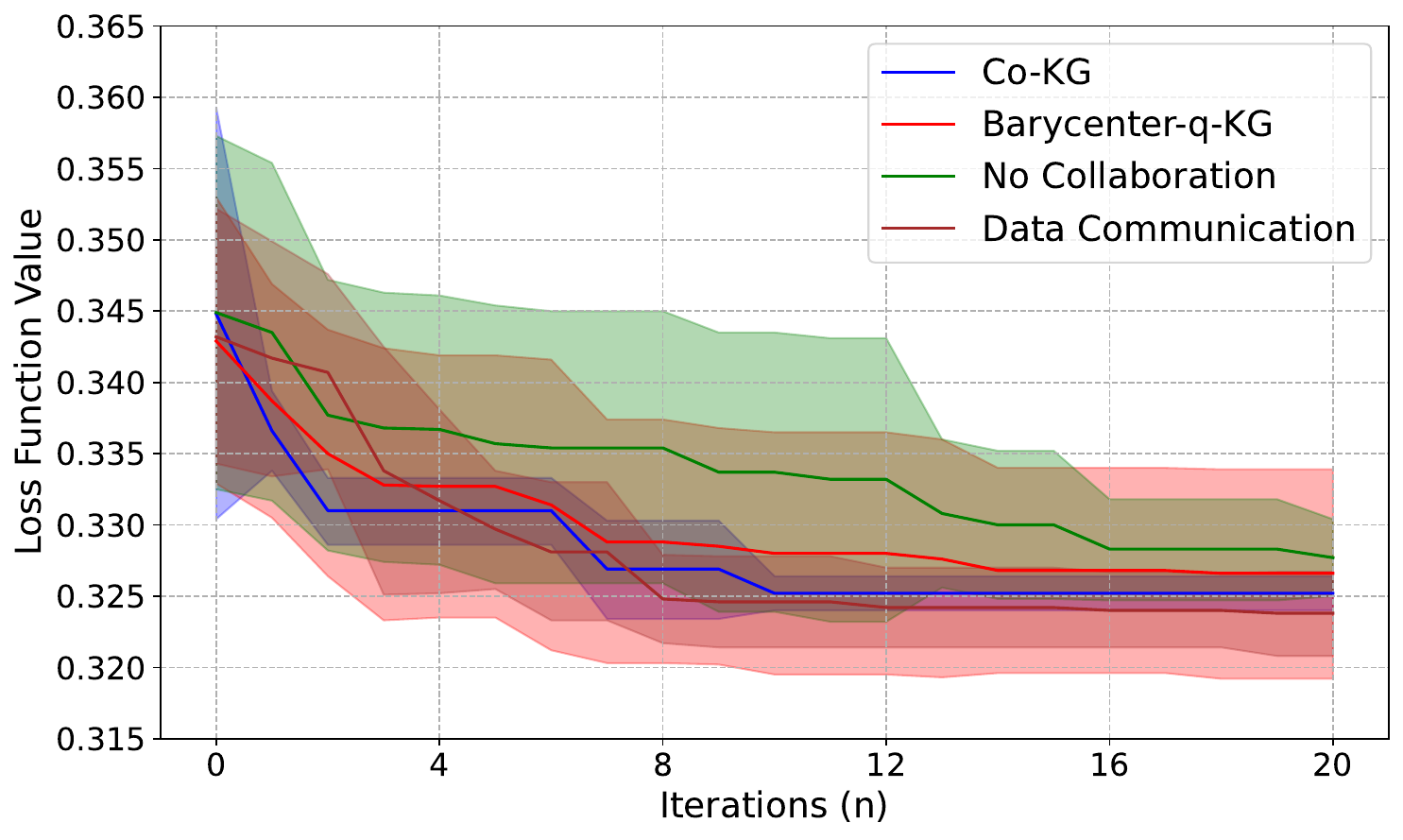}
        \caption{Loss function values in iterations with compared BO approaches on California Housing Dataset.}
        \label{fig:sub4}
    \end{minipage}
    %\caption{Overall caption for both images.}
    \label{fig:twocolfig}
\end{figure}

\subsection{Collaborative Framework Comparison}
We compare our framework using Co-KG with several baseline approaches including (i) the BO approach with the parallel Knowledge Gradient function focusing on the central model constructed by the Wasserstein Barycenter (Barycenter-$q$-KG); (ii) the BO approach with the Knowledge Gradient function implemented by each agent without collaboration (No Collaboration); and (iii) the BO approach with the parallel Knowledge Gradient function with permission to Data Communication (Data Communication).

Regarding the black-box optimization problem, we consider the task of hyperparameter tuning for learning 3-layer neural networks from data. Specifically, the decision variable is the learning rate and the hidden layer node size of the neural networks. The unknown objective function is the loss function we would minimize. The neural networks are learned from two datasets. The first dataset is related to the breast cancer\footnote{\url{https://archive.ics.uci.edu/dataset/17/breast+cancer+wisconsin+diagnostic}}, where the neural network is learned to predict the breast cancer classification label given predictive attributes. The second dataset is describes  California housing prices\footnote{\url{https://www.dcc.fc.up.pt/~ltorgo/Regression/cal\_housing.html}}, where the neural network is learned to predict median value of houses in different districts given demographic attributes. We do not impose noise on the observations in the real dataset.

We record the loss function values of training neural networks in Figure \ref{fig:sub3} and Figure \ref{fig:sub4}. The results provide the following insights. First, the Data Communication approach achieves the best performance with the smallest loss function values on both datasets, since this approach has permission to share data and therefore exploits the data most effectively to construct the GP model. Second, our proposed Co-KG approach achieves performance comparable to the Data Communication approach and outperforms the other two compared approaches on both datasets. This indicates the effectiveness of 1) collaboration among agents (as seen in the comparison with No Collaboration) and 2) considering the differentiation between agents (as seen in the comparison with Barycenter-$q$-KG). Lastly, when the unknown objective function is relatively simple to optimize (as in the Breast Cancer Dataset), distributed methods (i.e., no collaboration) can already achieve acceptable performance. In these scenarios, inefficient collaboration (Barycenter-$q$-KG) might decrease the performance of BO approaches.

\section{Conclusion}
We consider a collaborative framework for Bayesian optimization (BO) with data privacy, where multiple agents collect data to optimize an identical black-box objective function, without sharing their data. In our framework, agents share the Gaussian process (GP) models constructed with their own data, and a server builds a central GP model from the shared local GP models using the concept of the Wasserstein barycenter. We propose a general acquisition function that takes both the central model and local models into consideration and selects decision variables for agents in each iteration. We specifically focus on the knowledge gradient algorithm and propose a collaborative knowledge gradient (Co-KG) function. We establish the consistency of the BO approach based on Co-KG. To approximate Co-KG, we employ a Monte Carlo method and prove the consistency of this approximation as well. Additionally, we conduct numerical experiments to demonstrate that Co-KG outperforms other collaborative acquisition functions within our framework and achieves superior performance compared to other collaborative frameworks. We also show that our framework with Co-KG can achieve performance comparable to approaches that do not have data privacy concerns.

\section{Acknoweldgements}
James Anderson acknowledges funding from  NSF grants ECCS 2144634 and 2231350 and the Columbia Data Science Institute.

\bibliographystyle{IEEEtran}
\bibliography{main.bib}

\onecolumn

\section{Proof of Theoretical Results}
We prove Theorem 2 in the main text here. The result requires Assumption 1 as stated in the main text and repeated below: %\ja{[Please relabel this so that is Assumption 1]}
\setcounter{assumption}{0}
\begin{assumption}
    \begin{enumerate}\item[]
        \item The feasible set $\mathcal{X}$ is a compact set.
        \item Regarding the kernel function $K\left(x,x'\right)$, there exists a constant $\tau>0$ and a continuous function $\rho:\mathbb{R}^d \mapsto \mathbb{R}_{+}$ such that $K\left(x,x'\right) = \tau^2 \rho\left(x-x'\right)$. Moreover,
        \begin{enumerate}
            \item $\rho(|\delta|)=\rho(\delta)$, where $|\,\cdot\,|$ means taking the absolute value component-wise;
            \item $\rho(\delta)$ is decreasing in $\delta$ component-wise for $\delta \geqslant \mathbf{0}$;
\item $\rho(\mathbf{0})=1, \rho(\delta) \rightarrow 0$ as $\|\delta\| \rightarrow \infty$, where $\|\cdot\|$ denotes the Euclidean norm;
\item there exist some $0<C<\infty$ and $\varepsilon, u>$ 0 such that
$$
1-\rho(\delta) \leqslant \frac{C}{|\log (\|\delta\|)|^{1+\varepsilon}},
$$
for all $\delta$ such that $\|\delta\|<u$.
           
        \end{enumerate}
    \end{enumerate}
\end{assumption}
We first focus on the local GP model. Here we hide the index $n$ of each local GP. Recall that the objective function to optimize $f(x)$ is a GP model with the prior kernel function $K\left(x,x'\right)$. Specifically, we have the following proposition on the convergence of posterior kernel functions of $f\left(x\right)$.
\begin{proposition}[Proposition 1 of \cite{ding2022knowledge}]
    If the kernel function $K\left(x,x'\right)$ satisfies the conditions in Assumption \ref{assumption1}, then 
    \begin{equation*}
        \lim_{t\rightarrow\infty}\tilde{K}\left(x,x'\right)\stackrel{a.s.}{\longrightarrow} K^{\infty
        }\left(x,x'\right),
    \end{equation*}
    and the convergence is uniform. Here $\tilde{K}\left(x,x'\right)$ is the posterior kernel function after collecting $t$ observations, as defined in the main text, and $K^{\infty}\left(x,x'\right)$ is a function that does not depend on $t$.
\end{proposition}

%Furthermore, regarding the post

Next we provide a corollary regarding the posterior variance 
\begin{equation*}
    \operatorname{Var}\left [ f\left(x\right)\mid\tilde{\mathcal{S}} \right ] =\tilde{K}\left ( x,x \right ).
\end{equation*}
We note that, under Assumption \ref{assumption1}, there would be an accumulative point $x^{acc}\in\mathcal{X}$ for each local GP model. We here provide an asymptotic upper bound of $\operatorname{Var}\left [ f\left(x\right)\mid\tilde{\mathcal{S} }\right ]$ within an area centered at this accumulative point.

\begin{lemma}[Lemma 6 of \cite{ding2022knowledge}]
\label{lemma1}
    Under Assumption \ref{assumption1}, $\forall\epsilon>0$, we have 
    \begin{equation*}
        \limsup _{t \rightarrow \infty} \max _{x\in \mathcal{B}\left(x^{acc }, \epsilon\right)} \operatorname{Var}\left[f(x)\mid\tilde{\mathcal{S}}\right] \leqslant\tau^2\left[1-\rho^2(2 \epsilon \mathbf{1})\right],
    \end{equation*}
    where $\mathbf{1}$ is the vector of all ones with size $d\times 1$, $\mathcal{B}\left(x^{acc }, \epsilon\right)$ is the ball centered at $x^{acc}$ with radius $\epsilon$.
\end{lemma}

Recall that our Co-KG function is composed of one $q$-KG function with multiple decision variables as the input and multiple (regular) KG functions with one decision variable as the input. In the main text, we subtract the maximum posterior mean for simplification. That is, an equivalent definition of the KG function is
\begin{equation}
\label{eq.kg}
    \alpha_\text{KG}\left ( x \right )  = \mathbb{E}_{\tilde{\mathcal{S}}}\left[\max_{x'\in\mathcal{X}}\mathbb{E}\left[f\left ( x' \right )\mid \tilde{\mathcal{S}}_x^*\right]\right] - \max_{x'\in\mathcal{X}}\mathbb{E}\left[f\left ( x' \right )\mid \tilde{\mathcal{S}}\right].
\end{equation}
Here $\tilde{\mathcal{S}}$ represents the historical observations, and $\tilde{\mathcal{S}}_x^* = \tilde{\mathcal{S}}\cup\left\{\left(x,y(x)\right)\right\}$ is the updated dataset if an additional observation $y(x)$ is decided to be collected at the decision variable $x$. In this manner, regarding the posterior mean as the approximated objective function, KG represents the increment of the optimal value if an additional sample is collected at $x.$ Since the term $$\max_{x'\in\mathcal{X}}\mathbb{E}\left[f\left ( x' \right )\mid \tilde{\mathcal{S}}\right]$$ does not involve the decision variable $x$ to be optimized, we do not include it in the main text considering the limited length. A similar definition of the $q$-KG is in \cite{wu2016parallel}. When we prove the consistency of the collaborative BO procedure with Co-KG in this section, these terms are included. Furthermore, we note that the KG function is non-negative, to see this, we use the Jensen inequality:
\begin{equation*}
    \begin{aligned}
        \mathbb{E}\left[\max_{x'\in\mathcal{X}}\mathbb{E}\left[f\left(x'\right)\mid \tilde{\mathcal{S}}^*_{x}\right]\right] = &\mathbb{E}\left[\max_{x'\in\mathcal{X}}\mathbb{E}\left[\tilde{\mu}\left(x'\right)+\tilde{\sigma}\left ( x,x' \right ) \xi\right]\right]\\
        \geqslant&\max_{x'\in\mathcal{X}}\tilde{\mu}\left(x'\right)+\mathbb{E}\left[\tilde{\sigma}\left ( x,x' \right ) \xi\right]\\
    =&\max_{x'\in\mathcal{X}}\mathbb{E}\left[f\left ( x' \right )\mid \tilde{\mathcal{S}}\right],
    \end{aligned}
\end{equation*}
where $\tilde{\sigma}\left(x,x'\right)=\tilde{K}\left(x,x'\right)/\sqrt{\tilde{K} \left(x,x\right)+\sigma^2_\epsilon}.$The non-negativity based on the Jensen inequality also holds for the $q$-KG function with a similar argument \cite{wu2016parallel}. Since our Co-KG function is a weighted summation of a $q$-KG function and multiple regular KG functions, it is non-negative as well.

Regarding the KG function associated with each local GP model (\ref{eq.kg}), we have 
\begin{equation}
\label{eq.bound}
    \begin{aligned}
        \alpha_\text{KG}(x)=&\mathbb{E}\left[\max_{x'\in\mathcal{X}}\mathbb{E}\left[f\left(x'\right)\mid \tilde{\mathcal{S}}^*_{x}\right]\right]-\max_{x'\in\mathcal{X}}\mathbb{E}\left[f\left ( x' \right )\mid \tilde{\mathcal{S}}\right]\\ = &\mathbb{E}\left[\max_{x'\in\mathcal{X}}\mathbb{E}\left[\tilde{\mu}\left(x'\right)+\tilde{\sigma}\left ( x,x' \right ) \xi\right]\right]-\max_{x'\in\mathcal{X}}\mathbb{E}\left[f\left ( x' \right )\mid \tilde{\mathcal{S}}\right]\\
        \leqslant&\max_{x'\in\mathcal{X}}\tilde{\mu}\left(x'\right)+\mathbb{E}\left[\max_{x'\in\mathcal{X}
        }\tilde{\sigma}\left ( x,x' \right ) \xi\right]-\max_{x'\in\mathcal{X}}\mathbb{E}\left[f\left ( x' \right )\mid \tilde{\mathcal{S}}\right]\\
        =&\mathbb{E}\left[\max_{x'\in\mathcal{X}
        }\tilde{\sigma}\left ( x,x' \right ) \xi\right]\\
        \leqslant&\mathbb{E}\left [ \left|\xi\right| \right ]\max_{x'\in\mathcal{X}
        }\tilde{\sigma}\left ( x,x' \right )\\
        =&\sqrt{\frac{2}{\pi}}\max_{x'\in\mathcal{X}
        }\tilde{\sigma}\left ( x,x' \right )
        \end{aligned}
\end{equation}
Furthermore, we have that
\begin{equation}
\label{ine1}
    \begin{aligned}
       \tilde{\sigma}\left ( x,x' \right )&= \frac{\tilde{K} \left(x ,x'\right)}{\sqrt{\tilde{K} \left(x,x\right)+\sigma^2_\epsilon}}\\
       &\leqslant \sqrt{\frac{\tilde{K} \left(x ,x\right)\tilde{K} \left(x ',x'\right)}{\tilde{K} \left(x,x\right)+\sigma^2_\epsilon}}\\
       &\leqslant \sqrt{\frac{\tau^2\tilde{K} \left(x ,x\right)}{\sigma^2_\epsilon}},
    \end{aligned}
\end{equation}
where the last inequality comes from the fact that $\tilde{K}\left(x',x'\right)$ is a non-increasing sequence regarding $t$ and the conditions in Assumption \ref{assumption1}. Thus, from (\ref{eq.bound}) and (\ref{ine1}), we have
\begin{equation}
\label{kgbound}
    \alpha_\text{KG}(x)\leqslant\sqrt{\frac{2\tau^2\tilde{K} \left(x ,x\right)}{\pi\sigma^2_\epsilon}}
\end{equation}

Regarding the $q$-KG function, we have 
\begin{equation*}
%\label{eq.expression}
    \mathbb{E}\left [ f^c\left(x'\right)\mid \tilde{\mathcal{S}}^*_{\mathbf{x}} \right ] = \mu^c\left(x'\right)+\bm{\sigma}^c\left(\mathbf{x},x'\right)\bm{\xi},
\end{equation*}
where $\bm{\xi}\sim\mathcal{N}\left(\bm{0},\bm{I}_{N}\right)$. Additionally,
\begin{equation*}
    \bm{\sigma}^c\left(\mathbf{x},x'\right) = \bm{K}^c\left(\mathbf{x},x'\right) ^{\top}\left(\bm{D}\left(\mathbf{x}\right)^{\top}\right)^{-1},
\end{equation*}
where $\bm{D}\left(\mathbf{x}\right)$ is the Cholesky factor of the matrix $$\mathbf{\Sigma}\left(\mathbf{x}\right)=\begin{pmatrix}
K^c\left ( x_1,x_1 \right )   & \dots &K^c\left ( x_1,x_N \right ) \\
 \vdots & \ddots & \vdots\\
 K^c\left ( x_N,x_1 \right ) & \dots & K^c\left ( x_N,x_N \right )
\end{pmatrix}+\sigma^2_\epsilon \bm{I}_N.$$
Note that, 
\begin{equation*}
    \bm{\sigma}^c\left(\mathbf{x},x'\right)\bm{\xi}
    \sim 
    \mathcal{N}\left ( 0, \left \| \bm{\sigma}^c\left(\mathbf{x},x'\right) \right \|  \right ) .
\end{equation*}
With a similar argument as in (\ref{eq.bound}), the $q$-KG function is bounded by
\begin{equation*}
    \alpha_\text{$q$-KG}\left(\mathbf{x}\right)\leqslant \sqrt{\frac{2}{\pi}}\max_{x'\in\mathcal{X}}\left \| \bm{\sigma}^c\left(\mathbf{x},x'\right) \right \|.
\end{equation*}
Furthermore, 
\begin{equation}
\label{eq.qkgbound}
    \begin{aligned}
        &\left \| \bm{\sigma}^c\left(\mathbf{x},x'\right) \right \|^2 \\
        =&\bm{K}^c\left(\mathbf{x},x'\right) ^{\top}\mathbf{\Sigma}\left(\mathbf{x}\right)\bm{K}^c\left(\mathbf{x},x'\right)\\
        \leqslant&K^c\left(x',x'\right)   \\
        =&\left ( \frac{1}{N}\sum_{n=1}^N \tilde{K}_n\left(x',x'\right)  \right ) +\frac{1}{N} \sum_{n=1}^{N}\left(\tilde{\mu}_n\left ( x' \right )-\mu^c\left ( x' \right )  \right)^2,
    \end{aligned}
\end{equation}
where the last equality comes from the fact that the central GP is a 2-Wasserstein barycenter of local GPs. In this manner, we have connected the terms of the Co-KG function to the variances of local models. We here present a proposition regarding the conditional mean function $\tilde{\mu}\left(x\right)$.
\begin{proposition}[Proposition 2.9 in \cite{10.3150/18-BEJ1074}]
\label{prop.mu}
    Under Assumption \ref{assumption1}, the conditional mean function converges to $\mu^{\infty}(x)\doteq\mathbb{E}\left [ f(x)\mid \mathcal{F}_\infty \right ]$ uniformly in $x\in\mathcal{X}$ almost surely (a.s.). That is, 
    \begin{equation*}
        \mathbb{P}\left\{\sup _{x \in \mathcal{X}}\left|\tilde{\mu}(x )-\mu^{\infty}(x )\right| \rightarrow 0\right\}=1
    \end{equation*}
    as $t\rightarrow\infty$.

    Here $\mathcal{F}_\infty$ denotes the filtration of the dataset collection when the number of iterations approaches infinity.  
    
\end{proposition}

Next we consider a rescaled Co-KG function:
\begin{equation*}
    \label{eq.recokg}
    \alpha_\text{Co-KG}\left(\mathbf{x}\right) \doteq\mathbb{E}_{\tilde{\mathcal{S}}}\left\{\frac{\alpha^c\left(\mathbf{x}\right)}{\beta_t}+ \sum_{n=1}^N\alpha_n\left(x\right)\right\},
\end{equation*}
where $\alpha^c(\mathbf{x})$ is the $q$-KG function defined on the central GP and $\alpha_n(x)$ is the regular KG function defined on the $n$-th local GP model.
We provide an asymptotic property of Co-KG:
\begin{lemma}
\label{lemma.zero}    Under Assumption \ref{assumption1}, the limit inferior of Co-KG is 0. That is, 
    \begin{equation*}
        \liminf_{t\rightarrow\infty }\alpha_\text{Co-KG}(\mathbf{x}) = 0\quad\forall \mathbf{x}\in\mathcal{X}^{N}.
    \end{equation*}
\end{lemma}
\begin{proof}
We first provide the notation here. We denote the sequence of decision variables selected by maximizing Co-KG as 
\begin{equation*}
    \mathbf{x}(t) =\left ( x_1\left ( t \right ),x_2 \left ( t \right ),
\ldots, x_N\left ( t \right ) \right ) ,
\end{equation*}
where $x_n(t)$ denotes the decision variables for the $n$-th agent in the $t$-th iteration. In the main text, we did not emphasize the dependence on $t$ for notational simplicity. For each local model, note that the sequence of the selected $x_n(t)$'s for this local GP has an accumulative point $x^{acc}$. We denote the subsequence as $z_n(1),z_n(2),\ldots,z_n\left(t'\right)$ such that $z_n\left(t'\right)\rightarrow x^{acc}$. Based on Lemma \ref{lemma1}, we have 
    \begin{equation*}
        \limsup _{t \rightarrow \infty} \operatorname{Var}\left[f\left(z_n\left(t'\right)\right)\mid\tilde{\mathcal{S}}_n\right] \leqslant\tau^2\left[1-\rho^2(2 \epsilon \mathbf{1})\right].
    \end{equation*}
Furthermore, from (\ref{kgbound}), we have 
    \begin{equation*}
        \limsup _{t \rightarrow \infty}\alpha(z_n\left(t'\right))\leqslant\sqrt{\frac{2\tau^2\tilde{K} \left(z_n\left(t'\right) ,z_n\left(t'\right)\right)}{\pi\sigma^2_\epsilon}}\leqslant\limsup _{t \rightarrow \infty}\sqrt{\frac{2\tau^4}{\pi\sigma^2_\epsilon}}\left[1-\rho^2(2 \epsilon \mathbf{1})\right].
    \end{equation*}
    Let $\epsilon\rightarrow0$, we have that 
    \begin{equation}
    \label{eq.kglimit}
        \liminf_{t \rightarrow \infty} \alpha_\text{KG}\left(x_n(t)\right)\leqslant\liminf_{t \rightarrow \infty}\alpha_\text{KG}\left(z_n\left(t'\right)\right)\leqslant\limsup _{t \rightarrow \infty}\alpha_\text{KG}(z_n\left(t'\right)) =0\quad \forall n.
    \end{equation}
    The last equality holds because the KG function $\alpha(x)$ is non-negative.

    Then we look at the central GP and the associated $q$-KG function. Recall that, from (\ref{eq.qkgbound}), we have 
    \begin{equation*}
        \alpha_\text{$q$-KG}\left(\mathbf{x}\right)\leqslant \sqrt{\frac{2}{\pi}\left(\left ( \frac{1}{N}\sum_{n=1}^N \tilde{K}_n\left(x',x'\right)  \right ) +\frac{1}{N} \sum_{n=1}^{N}\left(\tilde{\mu}_n\left ( x' \right )-\mu^c\left ( x' \right )  \right)^2\right)}\quad \forall \mathbf{x},
    \end{equation*}
    where $x'$ denotes some maximizer. Note that $\tilde{K}_n\left(x',x'\right)\leqslant K\left(x',x'\right)$ is bounded, and each $\left(\tilde{\mu}_n\left ( x' \right )-\mu^c\left ( x' \right )  \right)^2$ is bounded as well because of Proposition \ref{prop.mu}. Thus, as $t\rightarrow \infty$, 
    \begin{equation*}
        \lim_{t\rightarrow\infty}\frac{\alpha_\text{$q$-KG}\left(\mathbf{x}\left(t\right)\right)}{\beta_t} 
    \end{equation*}
    since $\beta_t\rightarrow\infty$ from the definition in the main text. This further leads to 
    \begin{equation*}
        \liminf_{t\rightarrow\infty}\alpha_\text{Co-KG}\left(\mathbf{x}\left(t\right)\right)=0
    \end{equation*} 
    because of (\ref{eq.kglimit}). Recall that 
    \begin{equation*}
        \mathbf{x}\left(t\right)=\arg\max_{\mathbf{x}\in \mathcal{X}^{N}}\alpha_\text{Co-KG}\left(\mathbf{x}\right).
    \end{equation*}
    Thus, $\forall \mathbf{x}$, we have 
    \begin{equation*}
        0\leqslant\liminf_{t\rightarrow0}\alpha_\text{Co-KG}(\mathbf{x})\leqslant\liminf_{t\rightarrow0}\alpha_\text{Co-KG}(\mathbf{x}(t))=0
    \end{equation*}
\end{proof}

Now we prove that for all local GP models, the uncertainty at each decision variable (represented by the conditional variance) shrinks to zero. Specifically, we have the following lemma.  
\begin{lemma}
\label{lemmashrink}    Under Assumption \ref{assumption1}, $\forall n$ and $\forall x\in\mathcal{X}$, we have 
    \begin{equation}\label{result}
        \lim_{t\rightarrow\infty}\operatorname{Var}\left [ f\left ( x \right )\mid\tilde {\mathcal{S} }_n  \right ] =0\quad\forall x\in\mathcal{X},\forall n.
    \end{equation}
\end{lemma}
\begin{proof}
    Without loss of generality, we assume that 
    \begin{equation}
      \label{eq.nonzero}  \lim_{t\rightarrow\infty}\operatorname{Var}\left [ f\left ( \tilde{x} \right )\mid\tilde {\mathcal{S} }_1  \right ]=c>0.
    \end{equation}
    This limit exists because of Lemma \ref{lemma1}. Also, since $\tilde{K}\left(x,x'\right)$ is continuous and the convergence in Lemma \ref{lemma1} is uniform, $\operatorname{Var}\left [ f\left ( x \right )\mid\tilde {\mathcal{S} }_1  \right ]$ is continuous as well. Regarding this agent, we denote the current posterior function as $\tilde{\mu}\left(x\right) = \mathbb{E}\left[f\left ( x \right )\mid \tilde{\mathcal{S}}_1\right] $ and the posterior function with additional $x$ as $\tilde{\mu}^*\left(x'\right) = \mathbb{E}\left[f\left ( x' \right )\mid \tilde{\mathcal{S}}^*_{x}\right] $. Furthermore, we denote $\tilde{x}^*=\arg\max_{x\in\mathcal{X} }\tilde{\mu}(x) $. Also, let
    \begin{equation*}
        \begin{aligned}
            a_1 =& \tilde{\mu}\left(\tilde{x}^*\right)   \\
            b_1 =& \tilde{\sigma}\left(x, \tilde{x}^*\right) \\
            a_2 =& \tilde{\mu}\left(x\right) \\
            b_2 =& \tilde{\sigma}\left(x, x\right)  .\\
        \end{aligned}
    \end{equation*}
We consider the KG function associated with this agent:   
$$
\begin{aligned}
& \alpha_{\text{KG}}(x) \\
& =\mathbb{E}\left[\max_{x'\in\mathcal{X}}\mathbb{E}\left[f\left(x'\right)\mid \tilde{\mathcal{S}}^*_{x}\right]\right]-\max_{x'\in\mathcal{X}}\mathbb{E}\left[f\left ( x' \right )\mid \tilde{\mathcal{S}}\right] \\
& =\mathbb{E}\left[\max_{x'\in\mathcal{X}}\tilde{\mu}^*\left(x'\right)\right]-\max \left(\tilde{\mu}\left(\tilde{x}^*\right), \tilde{\mu}(x)\right) \\
& \geqslant \mathbb{E}\left[\max \left(\tilde{\mu}^*\left(\tilde{x}^*\right), \tilde{\mu}^*(x)\right)\right]-\max \left(\tilde{\mu}\left(\tilde{x}^*\right), \tilde{\mu}(x)\right) \\
& =\mathbb{E}\left[\max \left(\tilde{\mu}\left(\tilde{x}^*\right)+\tilde{\sigma}\left(x, \tilde{x}^*\right) \xi, \tilde{\mu}(x)+\tilde{\sigma}\left(x, x\right) \xi\right) \right]-\max \left(\tilde{\mu}\left(\tilde{x}^*\right), \tilde{\mu}(x)\right)\\
& =\mathbb{E}\left[\max \left(a_1+b_1 \xi, a_2+b_2 \xi\right)\right]-\max \left(a_1, a_2\right) \\
& = \begin{cases}\int_{-\infty}^{\frac{a_2-a_1}{b_1-b_2}}\left(a_2+b_2\xi\right) \phi(\xi) d\xi+\int_{\frac{a_2-a_1}{b_1-b_2}}^{\infty}\left(a_1+b_1\xi\right) \phi(\xi) d\xi-\max \left(a_1, a_2\right), & \text { if } b_2 \leqslant b_1 \\
\int_{-\infty}^{\frac{a_2-a_1}{b_1-b_2}}\left(a_1+b_1\xi\right) \phi(\xi) d\xi+\int_{\frac{a_2-a_1}{b_1-b_2}}^{\infty}\left(a_2+b_2\xi\right) \phi(\xi) d\xi-\max \left(a_1, a_2\right), & \text { if } b_1<b_2\end{cases} \\
& = \begin{cases}a_2 \Phi\left(\frac{a_2-a_1}{b_1-b_2}\right)-b_2 \phi\left(\frac{a_2-a_1}{b_1-b_2}\right)+a_1\left(1-\Phi\left(\frac{a_2-a_1}{b_1-b_2}\right)\right)+b_1 \phi\left(\frac{a_2-a_1}{b_1-b_2}\right)-\max \left(a_1, a_2\right), & \text { if } b_2 \leqslant b_1 \\
a_1 \Phi\left(\frac{a_2-a_1}{b_1-b_2}\right)-b_1 \phi\left(\frac{a_2-a_1}{b_1-b_2}\right)+a_2\left(1-\Phi\left(\frac{a_2-a_1}{b_1-b_2}\right)\right)+b_2 \phi\left(\frac{a_2-a_1}{b_1-b_2}\right)-\max \left(a_1, a_2\right), & \text { if } b_1<b_2\end{cases} \\
& =a_2 \Phi\left(\frac{a_2-a_1}{\left|b_1-b_2\right|}\right)+a_1\left(1-\Phi\left(\frac{a_2-a_1}{\left|b_1-b_2\right|}\right)\right)+\left|b_1-b_2\right| \phi\left(\frac{a_2-a_1}{\left|b_1-b_2\right|}\right)-\max \left(a_1, a_2\right) \\
& =-\left|a_2-a_1\right| \Phi\left(\frac{-\left|a_2-a_1\right|}{\left|b_1-b_2\right|}\right)+\left|b_1-b_2\right| \phi\left(\frac{\left|a_2-a_1\right|}{\left|b_1-b_2\right|}\right),
\end{aligned}
$$
where $\Phi$ is the standard normal distribution function and $\phi$ is its density function.

Let $g(s, t):=t \phi(s / t)-s \Phi(-s / t)$. Then 1) $g(s, t)>0$ for all $s \geqslant 0$ and $t>0$; 2) $g(s, t)$ is strictly decreasing in $s \in[0, \infty)$ and strictly increasing in $t \in(0, \infty)$; and 3) $g(s, t) \rightarrow 0$ as $s \rightarrow \infty$ or as $t \rightarrow 0$. See more details in \cite{scott2011correlated,ding2022knowledge}. By letting $x=\tilde{x}$, which is defined in (\ref{eq.nonzero}), we have that $\liminf_{t\rightarrow\infty}\left|b_1-b_2\right|\geqslant c'$ for some constant $c'>0$. Meanwhile, based on Proposition \ref{prop.mu}, we have $\limsup_{t\rightarrow\infty}\left|a_2-a_1\right|\leqslant r'$ for some constant $r'<\infty$. Thus, 
\begin{equation*}
    \liminf_{t\rightarrow\infty} \alpha _{\text{KG}}\left(\tilde{x}\right)\geqslant g\left(r',c'\right)>0.
\end{equation*}
A similar argument is in Theorem 5.6 in \cite{scott2011correlated}. On the other hand, since $q$-KG and KG are non-negative, we have $$\liminf_{t\rightarrow\infty}\alpha_{\text{Co-KG}}(\mathbf{x})\geqslant \liminf_{t\rightarrow\infty}\alpha_{\text{KG}}(x)>0,$$
where $\tilde{\mathbf{x}}$ includes $\tilde{x}$ as one component. This provides a contradiction with Lemma \ref{lemma.zero}, which proves the result in (\ref{result}).

\end{proof}

Based on Lemma \ref{lemmashrink}, we prove Theorem 2 in the main text. 
\begin{proof}
    For each local GP, we have 
    \begin{equation*}
        \mathbb{E}\left[\tilde{\mu}_n\left(x\right)-f\left ( x \right ) \right]^2=\operatorname{Var} \left [ f\left ( x \right ) \mid\tilde{\mathcal{S}}_n\right ] \rightarrow 0
    \end{equation*}
    from Lemma \ref{lemmashrink}. In addition, based on Proposition \ref{prop.mu}, we have that 
    \begin{equation*}
        \mathbb{P}\left\{\sup _{x \in \mathcal{X}}\left|\tilde{\mu}_n(x)-\mu_n^{\infty}(x)\right| \rightarrow 0\right\}=1\quad\forall n
    \end{equation*}
    as $t\rightarrow\infty$. Thus, $\mu_n^\infty\left ( x \right ) \stackrel{a.s.}{=}f(x)$ and 
    \begin{equation*}
        \mathbb{P}\left\{\sup _{x \in \mathcal{X}}\left|\tilde{\mu}_n(x)-f(x)\right| \rightarrow 0\right\}=1\quad\forall n
    \end{equation*}
    as $t\rightarrow\infty$. Since kernel function $K\left(x,x'\right)$ is continuous, the conditional mean function $\tilde{\mu}_n(x)$, as well as $f(x)$, is continuous as well. Therefore, for each local model, the optimizer submitted satisfies that 
    \begin{equation*}
        \lim _{t\rightarrow\infty}f\left ( \hat{x}^*_n \right ) =\max_{x\in\mathcal{X} }f(x).
    \end{equation*}
    We refer to \cite{van2000asymptotic} for a detailed proof. In this manner, we have the consistency of the collaborative BO procedure with Co-KG (as summarized in Algorithm 1 in the main text):
    \begin{equation*}
        \lim _{t\rightarrow\infty}f\left ( \hat{x}^* \right ) =\max_{x\in\mathcal{X} }f(x).
    \end{equation*}
\end{proof}

\section{Additional Experiments}
We present additional experiments here, including 1) the effects of different selections of $\beta_t$ in the Co-KG function, 2) the effects of the discretization of the feasible set $\mathcal{X}$, 3) the comparison between the Co-KG procedure with different numbers of agents. Our experiments were conducted with Botorch \cite{balandat2020botorch} and Python 3.9 on a computer equipped with two AMD Ryzen Threadripper 3970X 32-Core Processors, 128 GB memory, and a Nvidia GeForce RTX A6000 GPU with 48GB of RAM. The implementation will be released once accepted.

\subsection{Feasible Set Discretization}
\label{sec.exp2}
We discuss the impacts of feasible set discretization here. Regarding the hyperparameter, we set $\beta_t = \log\left(2t+1\right)$. We have $N=4$ agents. We normalize the feasible set to $\mathcal{X}=[0, 1]^2$ and discretize the feasible set using 1) $10\times10$, 2) $20 \times 20$, and 3) $30\times 30$ uniform mesh grids. We also include the results associated with the parallel BO approach without data privacy concerns ($q$-KG) for comparison, and the procedure is indicated by ``Data Communication''.

Regarding the black-box optimization problem, we consider minimizing the validation loss of training a neural network. Specifically, the decision variable is the learning rate and the hidden layer node size of the neural networks. The unknown objective function is the validation loss we would minimize. The dataset is about California housing\footnote{\url{https://www.dcc.fc.up.pt/~ltorgo/Regression/cal\_housing.html}}, where the neural network is learned to predict median value of houses in different districts given demographic attributes. We do not impose noise on the observations in the real dataset. 
\begin{figure}[!ht]
    \centering
        \includegraphics[width=0.88\textwidth]{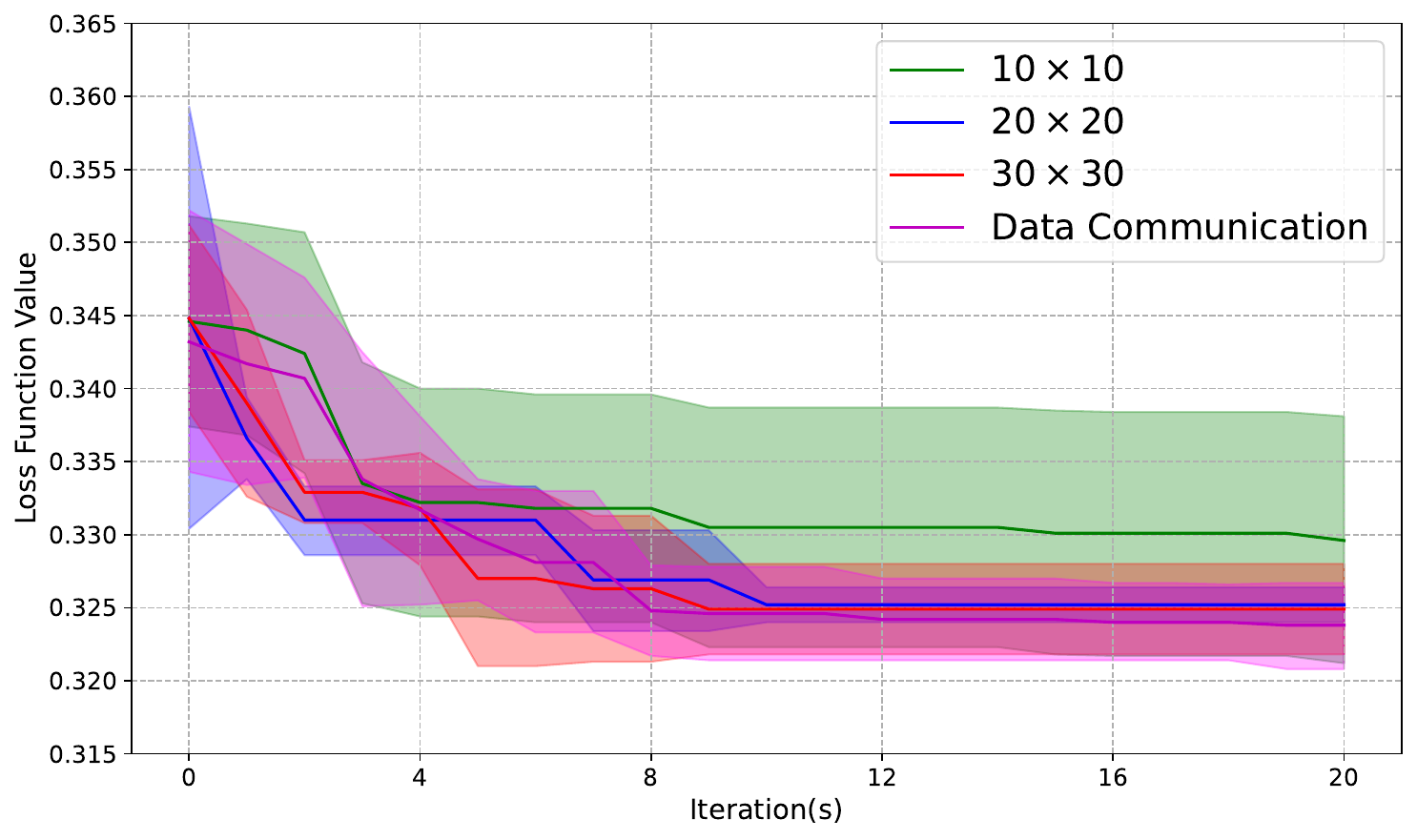}
    \caption{Validation loss in iterations with different discretization sizes on California Housing Dataset.}
    \label{fig3}
\end{figure}

The experimental results are included in Figure \ref{fig3}, and provide the following insights: First, increasing the granularity of mesh grid enhances the performance of the collaborative BO approach with Co-KG, with more accurate central model construction and more flexible decision variable selections. Second, when the grid is precise enough ($20\times20$ and $30\times 30$), our approach is comparable with that without data privacy concerns, which is also supported by the experiments in the main text. Lastly, although increasing the granularity from $20\times20$ to $30\times 30$ enhances the performance of Co-KG, the enhancement is not significant. On the other hand, increasing the granularity significantly increases the computational burdens, we include the average running time per iteration in Table \ref{tab:my_label}. From the table, we observe that the time complexity for discretization is nearly $O(N^{2})$, leading to heavy computational cost if we set a dense discretization for feasible set. Furthermore, we admit that the discretization could be time-varying regarding different iterations and adaptive to the collaborative optimization procedure, while the detailed discussions are left for future work.

\begin{table}[!ht]
    \centering
    \begin{tabular}{c|c}
    \hline\hline
      Discretization mesh grid   & Computational time (s)  \\\hline
      $10\times10$   &   2.02   \\
     $20\times20$    &    6.68  \\
      $30\times30$   &   16.20      \\\hline\hline
    \end{tabular}
    \caption{Computational time with different discretization strategies.}
    \label{tab:my_label}
\end{table}

\subsection{Agent Number Comparison}
We compare the effects of the number of agents. The experimental settings are the same in Section \ref{sec.exp2} with the mesh grid fixed to be $20\times 20$. We consider $N=2,4,8$ in our experiments.

\begin{figure}[!ht]
        \centering
        \includegraphics[width=0.88\textwidth]{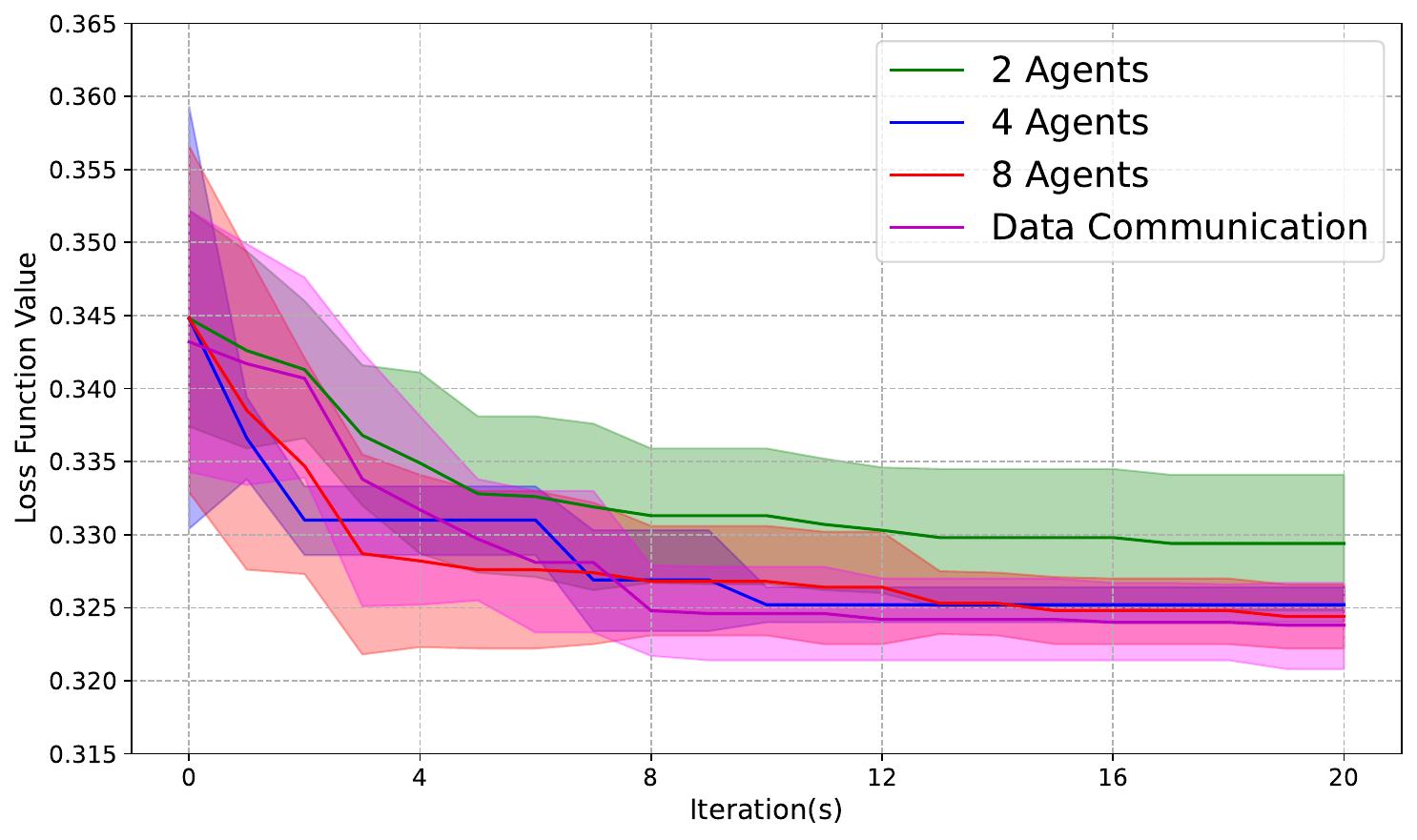}
    \caption{Validation loss in iterations with different number of agents on California Housing Dataset.}
    \label{fig4}
\end{figure}
The experimental results in Figure \ref{fig4} reveal the following insights. First, when the number of agents is low ($N=2$), the performance of Co-KG is suboptimal. Second, comparing the results between $N=4$ and $N=8$, we observe that having more agents does not necessarily lead to better performance, especially during the initial iterations. This is because, with more agents, some may be initialized in less promising regions, which negatively impacts collaboration. The Co-KG function currently assigns equal weights to all agents, which is sensitive to suboptimal observations collected by some agents. Future work could explore assigning different weights in the collaborative acquisition function to enhance algorithm robustness. Finally, as the number of iterations increases, the procedure with $N=8$ agents slightly outperforms that with $N=4$, but the difference is not substantial. Both procedures tend to stabilize without significant improvements, due to the effects of discretization. This suggests that the optimal number of agents may also depend on the level of discretization, a topic that is left for future discussion.

\end{document}